\documentclass[11pt]{article}

\usepackage{url}
\usepackage{smile}
\usepackage{anyfontsize}
\usepackage[inline]{enumitem}
\usepackage{xcolor}
\usepackage{hyperref}

\usepackage{kpfonts}
\DeclareMathAlphabet{\mathsf}{OT1}{cmss}{m}{n}
\SetMathAlphabet{\mathsf}{bold}{OT1}{cmss}{bx}{n}

\usepackage{fullpage}

\newcommand{\rk}{\text{rank}}
\newcommand{\Fr}{\text{F}}
\newcommand{\RN}{\text{RN}}

\newcommand{\Jac}{\text{Jac}}
\newcommand{\Net}{\text{Net}}
\newcommand{\removed}[1]{}
\def\[#1\]{{\small$#1$}}

\newcommand{\note}[1]{{\color{blue}{#1}}}

\usepackage[T1]{fontenc}
\usepackage{etoolbox}
\usepackage{setspace}
\AtBeginDocument{%
	\addtolength\abovedisplayskip{-0.05\baselineskip}%
	\addtolength\belowdisplayskip{-0.05\baselineskip}%
	\addtolength\abovedisplayshortskip{-0.05\baselineskip}%
	\addtolength\belowdisplayshortskip{-0.05\baselineskip}%
}

\begin{document}

\title{\LARGE \bf On Tighter Generalization Bounds for Deep Neural Networks: CNNs, ResNets, and Beyond}

\author{\normalsize Xingguo Li, Junwei Lu, Zhaoran Wang, Jarvis Haupt, and Tuo Zhao\thanks{Xingguo Li is affiliated with the Computer Science Department at Princeton University, Princeton, NJ 08540; Junwei Lu is affiliated with Department of Operations Research and Financial Engineering at Princeton University, Princeton, NJ 08544; Zhaoran Wang is affiliated with Department of Industrial Engineering and Management Sciences at Northwestern University, Evanston, IL 60208; Jarvis Haupt is affiliated with Department of Electrical and Computer Engineering at University of Minnesota, Minneapolis, MN 55455; Tuo Zhao is affiliated with School of Industrial and Systems Engineering at Georgia Institute of Technology, Atlanta, GA 30332; Tuo Zhao is the corresponding author; Email: {\tt xingguol@cs.princeton.edu},{\tt tourzhao@gatech.edu}}}

\date{}

\maketitle

\begin{abstract}
	We establish a margin based data dependent generalization error bound for a general family of deep neural networks in terms of the depth and width, as well as the Jacobian of the networks. Through introducing a new characterization of the Lipschitz properties of neural network family, we achieve significantly tighter generalization bounds than existing results. Moreover, we show that the generalization bound can be further improved for bounded losses. Aside from the general feedforward deep neural networks, our results can be applied to derive new bounds for popular architectures, including convolutional neural networks (CNNs) and residual networks (ResNets).  When achieving same generalization errors with previous arts, our bounds allow for the choice of larger parameter spaces of weight matrices, inducing potentially stronger expressive ability for neural networks. Numerical evaluation is also provided to support our theory.
\end{abstract}

\section{Introduction}

We aim to provide a theoretical justification for the enormous success of deep neural networks (DNNs) in real world applications \cite{he2016deep,collobert2011natural,goodfellow2016deep}. 
In particular, our paper focuses on the generalization performance of a general class of DNNs. 
The generalization bound is a powerful tool to characterize the predictive performance of a class of learning models for unseen data.  Early studies investigate the generalization ability of  shallow neural networks with no more than one hidden layer \cite{bartlett1998sample,anthony2009neural}. More recently, studies on the generalization bounds  of  deep neural networks have received increasing attention \cite{dinh2017sharp,bartlett2017spectrally,golowich2017size,neyshabur2015norm,neyshabur2017pac}. There are two major questions of our interest in these analysis of the generalization bounds:
\begin{itemize}[leftmargin=0.2in,itemsep=0.0in]
	\item ({\it Q1}) {\it \textbf{Can we establish tighter generalization error bounds for deep neural networks in terms of the network dimensions and structure of the weight matrices?}}
	\item ({\it Q2}) {\it \textbf{Can we develop generalization bounds for neural networks with special architectures?}}
\end{itemize}
%


For (Q1), \cite{neyshabur2015norm,bartlett2017spectrally,neyshabur2017pac,golowich2017size} have established results that characterize the generalization bounds in terms of the depth $D$ and width $p$ of networks and norms of rank-$r$ weight matrices. For example, \cite{neyshabur2015norm} provide an exponential bound on $D$ based on $\nbr{W_{d}}_{\Fr}$ (Frobenius norm), where $W_d$ is the weight matrix of $d$-th layer; \cite{bartlett2017spectrally,neyshabur2017pac} provide a polynomial bound on $p$ and $D$ based on $\nbr{W_{d}}_2$ (spectral norm) and $\nbr{W_d}_{2,1}$ (sum of the Euclidean norms for all rows of $W_{d}$). \cite{golowich2017size} provide a nearly size independent bound based on $\nbr{W_{d}}_{\Fr}$. Nevertheless, the generalization bound that depends on the product of norms may be too loose, especially those on those other than the spectral norm. For example, $\nbr{W_{d}}_{\Fr}$ ($\nbr{W_d}_{2,1}$) is in general $\sqrt{r}$ ($r$) times larger than $\nbr{W_d}_{2}$. Given $m$ training data points, \cite{bartlett2017spectrally} and \cite{neyshabur2017pac} demonstrate generalization error bounds as $\tilde{\cO} (\prod_{d} \nbr{W_{d}}_2 \sqrt{D^3 p r/m})$, and \cite{golowich2017size} achieve a bound $\tilde{\cO}(\prod_{d} \nbr{W_{d}}_{\Fr} \min(m^{-1/4}, \sqrt{D/m}))$, where $\tilde{\cO}(\cdot)$ represents the rate by ignoring logarithmic factors. In comparison, we show a tighter margin based bageneralization error bound as $\tilde{\cO}(\nbr{\text{Jacobian}}_2 \sqrt{Dpr/m})$, which is significantly smaller than existing results based on the product of norms.  Our bound is achieved based on a new Lipschitz analysis for DNNs in terms of both the input and weight matrices. Moreover, numerical results also support that our derived bound is significantly tighter than the existing bounds. Some recent result achieved results that is free of the linear dependence on the weight matrix norms by considering networks with bounded outputs \cite{zhou2018understanding}. We can achieve similar results using bounded loss functions as discussed in Section~\ref{sec:indep}. 

We notice that some recent results characterize the generalization bound in more structured ways, e.g., by considering specific error-resilience parameters \cite{arora2018stronger}, which can achieve empirically improved generalization bounds than existing ones based on the norms of weight matrices. However, it is not clear how the weight matrices explicitly control these parameters, which makes the results less interpretable. More recently, localized capacity analysis is conducted by considering the parameter space that (stochastic) gradient descent converges to \cite{allen2018learning,cao2019generalization}. However, they study the over-parameterization regime with extremely wide networks, which is not required in our analysis, thus their results are not directly comparable here. We summarize the result of norm based generalization bounds and our results in Table~\ref{table:compare}, as well as the results when $\nbr{W_{d}}_2=1$ for more explicit comparison in terms of the network sizes (i.e, depth and width). Further numerical comparison is provided in Section~\ref{sec:exp_all}.

For (Q2), we consider two widely used architectures: convolutional neural networks (CNNs) \cite{krizhevsky2012imagenet} and residual networks (ResNets) \cite{he2016deep} to demonsrate. By taking their structures of weight matrices into consideration, we provide tight characterization of their resulting capacities. In particular, we consider orthogonal filters and normalized weight matrices, which show good performance in both optimization and generalization \cite{mishkin2015all,xie2017all}. This is closely related with normalization frameworks, e.g., batch normalization \cite{ioffe2015batch} and layer normalization \cite{ba2016layer}, which have achieved great empirical performance \cite{liu2017sphereface,he2016deep}. Take CNNs as an example. By incorporating the orthogonal structure of convolutional filters, we achieve \[\tilde{\cO}\big({  \rbr{\frac{k}{s}}^{\frac{D}{2}}\sqrt{D k^2} }/{\sqrt{m}} \big)\], while \cite{bartlett2017spectrally,neyshabur2017pac} achieve \[\widetilde{\cO}\big({  \rbr{\frac{k}{s}}^{\frac{D-1}{2}} \sqrt{D^3 p^2} }/{\sqrt{m}} \big)\] and \cite{golowich2017size} achieve \[\widetilde{\cO}\Big( {p^{\frac{D}{2}}} \min\Big\{\frac{1}{\sqrt[4]{m}}, \sqrt{\frac{D}{m}} \Big\} \Big)\] ($\rk(W_d)=p$ in CNNs), where $k$ is the filter size that satisfies $k \ll p$ and $s$ is stride size that is usually of the same order with $k$; see Section~\ref{sec:cnn} for details. Here we achieve stronger results in terms of both depth $D$ and width $p$ for CNNs, where our bound only depend on $k$ rather than $p$. Analogous improvement is also attained for ResNets. In addition, we consider some widely used operations for width expansion and reduction, e.g., padding and pooling, and show that they do not increase the generalization bound. 

\begin{table*}[!t]
	\begin{center}
		\caption{Comparison of existing results with ours on norm based capacity bounds for DNNs. For ease of illustration, we suppose the upper bound of input norm $R$ is a generic constant. We use $B_{d,2}$, $B_{d,\Fr}$, and $B_{d,2\rightarrow 1}$ as the upper bounds of $\nbr{W_{d}}_2$, $\nbr{W_{d}}_{\Fr}$, and $\nbr{W_{d}}_{2,1}$ respectively. For notational convenience, we denote $\Gamma \leq \prod_{d=1}^D \nbr{W_d}_2$, $g_{\gamma}$ defined in \eqref{eqn_dnn:ramp}, $B^{\Jac}_{1:D}$ and $B^{\Jac}_{\backslash d,2}$ defined in Theorem~\ref{thm:tighter_upperbd}, and suppose the width $p_{d} = p$ for all layers $d = 1,\ldots,D$. We further show the results when $\nbr{W_{d}}_2=1$ for all $d=1,\ldots,D$, where $\nbr{W_{d}}_{\Fr} = \Theta(\sqrt{r})$ and $\nbr{W_{d}}_{2,1} = \Theta(r)$ in generic scenarios.
		}
		{
			\renewcommand{\arraystretch}{1.5}
			\begin{tabular}{c|c|c}
				\Xhline{1 pt}
				Capacity Bound & Original Results & $\nbr{W_{d}}_2=1$  \\
				\hline
				\cite{neyshabur2015norm} & $\cO\rbr{\frac{ 2^D \cdot \Pi_{d=1}^{D} B_{d,\Fr} }{\gamma \sqrt{m}}}$ & $\cO\rbr{\frac{ 2^D \cdot  r^{{D}/{2}} }{\gamma \sqrt{m}}}$ \\
				\hline
				\cite{bartlett2017spectrally} & ${\cO}\rbr{\frac{ \Pi_{d=1}^{D} B_{d,2} \cdot \log \rbr{p} }{\gamma \sqrt{m}} \rbr{\sum_{d=1}^{D} \frac{B_{d,2\rightarrow 1}^{{2}/{3}}}{B_{d,2}^{{2}/{3}}} }^{{3}/{2}}}$ & $\widetilde{\cO}\rbr{\frac{  \sqrt{D^3 pr} }{\gamma \sqrt{m}}}$ \\
				\hline
				\cite{neyshabur2017pac} & ${\cO}\rbr{\frac{ \Pi_{d=1}^{D} B_{d,2} \cdot \log \rbr{Dp} }{\gamma \sqrt{m}} \sqrt{D^2 p \sum_{d=1}^{D} \frac{B_{d,\Fr}^{2}}{B_{d,2}^{2} } } }$ & $\widetilde{\cO}\rbr{\frac{  \sqrt{D^3 pr} }{\gamma \sqrt{m}}}$ \\
				\hline
				\cite{golowich2017size} & ${\cO}\rbr{\frac{\Pi_{d=1}^{D} B_{d,\Fr}}{\gamma } \cdot \min \Bigg\{\frac{ \sqrt{ \log \frac{ \Pi_{d=1}^{D} B_{d,\Fr} }{ \Gamma } } }{\sqrt[4]{m}}, \sqrt{\frac{D}{m}} \Bigg\} }$ & $\widetilde{\cO}\rbr{ \frac{\sqrt{r^{D} \cdot D}}{\gamma \sqrt[4]{m}} }$ \removed{$\widetilde{\cO}\rbr{ \frac{r^{{D}/{2}}}{\gamma} \min\cbr{\frac{1}{\sqrt[4]{m}}, \sqrt{\frac{D}{m}}} }$} \\
				\hline
				Our results & \note{${\cO} \Bigg( \frac{ B_{1:D}^{\Jac} \sqrt{D pr} \cdot \log \rbr{\frac{B^{\Jac}_{\backslash d,2} \cdot \sqrt{Dm/r} \cdot \max_{d} B_{d,2}}{\gamma \cdot \sup g_{\gamma} \rbr{f\rbr{\cW_{D},x}}} } }{\gamma \sqrt{m}} \Bigg)$} & \note{$\tilde{\cO}\rbr{\frac{\sqrt{D pr} }{\gamma \sqrt{m}}}$} \\
				\Xhline{1 pt}
			\end{tabular}\label{table:compare}
		}
	\end{center}
\end{table*}

Our tighter bounds result in potentially stronger expressive power, hence higher training and testing accuracy for the DNNs. In particular, when achieving the same order of generalization errors, we allow the choice of a larger parameter space with deeper/wider networks and larger matrix spectral norms. We further show numerically that a larger parameter space can lead to better empirical performance. Quantitative analysis for the expressive power of DNNs is of great interest on its own, which includes (but not limited to) studying how well DNNs can approximate general class of functions and distributions \cite{cybenko1989approximation,hornik1989multilayer,funahashi1989approximate,barron1993universal,barron1994approximation, lee2017ability,petersen2017optimal,hanin2017approximating}, and quantifying the computation hardness of learning neural networks; see e.g., \cite{shamir2016distribution,eldan2016power,song2017complexity}. We defer our investigation  to future efforts.

\noindent {\bf Notation.} Given an integer $n >0$, we define $[n] = \cbr{1,\ldots,n}$. Given a matrix $A \in \RR^{n \times m}$, we denote $\nbr{A}$ as a generic norm, $\nbr{A}_2$ as the spectral norm, $\nbr{A}_{\Fr}$ as the Frobenius norm, and $\nbr{A}_{2,1} = \sum_{i=1}^{n} \nbr{A_{i*}}_2$. We write $\cbr{a_i}_{i=1}^{n} = \cbr{a_1,\ldots,a_n}$ as a set containing a sequence of size $n$. Given two real values $a,b \in \RR^+$, we write $a \lesssim(\gtrsim) b$ if $a \leq(\geq) c b$ for some generic constant $c>0$. 
We use $\cO\rbr{\cdot}$, $\Theta\rbr{\cdot}$, and $\Omega\rbr{\cdot}$ to denote limiting behaviors ignoring constants, and $\tilde{\cO}\rbr{\cdot}$, $\tilde{\Theta}\rbr{\cdot}$ and $\tilde{\Omega}\rbr{\cdot}$ to further ignore logarithms. 

\section{Preliminaries}\label{sec:background}

We provide a brief description of the DNNs. Given an input $x \in \RR^{p_0}$, the output of a $D$-layer network is defined as $f\rbr{\cW_{D},x} = f_{W_{D}} \rbr{\cdots f_{W_{1}}\rbr{x} } \in  \RR^{p_D}$, where $f_{W_{d}}(y) = \sigma_d \rbr{W_{d} \cdot y} : \RR^{p_{d-1}} \rightarrow \RR^{p_{d}}$ with an entry-wise activation function $\sigma_d(\cdot)$. We specify $\sigma_d$ as the rectified linear unit (ReLU) activation \cite{nair2010rectified}. The extension to more general activations, e.g., Lipschitz continuous functions, is straightforward. 
We also introduce some additional notations. Given any two layers $i,j \in [D]$ and input $x$, we denote $J_{i:j}^x$ as the Jacobian from layer $i$ to layer $j$, i.e., $f_{W_{j}} \rbr{\cdots f_{W_{i}}\rbr{x} } = J_{i:j}^x \cdot x$. For convenience, we denote $f_{W_{i}}\rbr{x} = J_{i,i}^x \cdot x$ when $i=j$ and denote $J_{i:j}^x=I$ when $i>j$. 

Then we denote DNNs with bounded Jocobian with weight matrices $\cW_{D} = \cbr{W_d}_{d=1}^D$ and ranks as
\begin{align}
	\hspace{-0.05in}\cF_{D,\Jac} = \big\{f\rbr{\cW_{D},x}~|~ \forall d \in [D],W_d \in \cW_{D}, \rk\rbr{W_d} \leq r_d, \text{sup}_{\cW_{D}} \nbr{J_{i,j}^x}_2 \leq B^{\Jac,x}_{i:j} \big\}, \hspace{-0.05in} \label{eqn_dnn:nn_Dl}
\end{align}
where $x \in \RR^{p_0}$ is an input, and $\{B^{\Jac,x}_{i:j} \}$ are real positive constants. 
For convenience, we also denote $\nbr{W_d}_2 \leq B_{d,2}$ and $\nbr{W_d}_{\Fr} \leq B_{d,\Fr}$ for all $d \in [D]$ for weight matrices when necessary.

Given a loss function $g(\cdot,\cdot)$, we denote a class of loss functions measuring the discrepancy between a DNN's output $f\rbr{\cW_{D},x}$ and the corresponding observation $y \in \cY_{m}$ for a given input $x \in \cX_{m}$ as
\begin{align*}
	\cG \rbr{\cF_{D,\Jac} } = \big\{g(f\rbr{\cW_{D},x},y) \in \RR~|~x \in \cX_{m}, y \in \cY_{m}, f\rbr{\cdot,\cdot} \in \cF_{D,\Jac} \big\},
\end{align*}
where the sets of bounded inputs $\cX_{m}$ and the corresponding observations $\cY_{m}$ are
\begin{align*}
	\cX_{m} = \cbr{x_i \in \RR^{p_0} \mid \nbr{x_i}_2 \leq R~\text{for all}~i \in [m]} ~\text{and}~~
	\cY_{m} = \cbr{y_i \in \sbr{p_{D}} ~~\text{for all}~~i \in [m]}.
\end{align*}
Then the empirical Rademacher complexity (ERC) of $\cG \rbr{\cF_{D,\Jac} }$ given $\cX_{m}$ and $\cY_{m}$ is
\begin{align}
	\cR_m \rbr{\cG \rbr{\cF_{D,\Jac} }} =  \mathop{\EE}_{\epsilon \in \{ \pm 1 \}^m } \sbr{ \sup_{f \in \cF_{D,\Jac}} \bigg| \frac{1}{m} \sum_{i=1}^{m} \epsilon_i  \cdot  g\rbr{f\rbr{\cW_{D},x_i},y_i} \bigg| }, \label{eqn_dnn:rademacher}
\end{align}
where $\{ \pm 1 \}^m \in \RR^{m}$ is the set of vectors only containing entries $+1$ and $-1$, and $\epsilon \in \RR^{m}$ is a vector with Rademacher entries, i.e., $\epsilon_i = +1$ or $-1$ with equal probabilities.

Take the classification as an example. For multi-class classification, suppose $p_{D} = N_{\rm class}$ is the number of classes. Consider $g$ with bounded outputs, namely the {\it ramp risk}. Specifically, for an input $x$ belonging to class $y \in [N_{\rm class}]$, we denote $\nu_{\cW_{D}}^{x,y} = \rbr{f\rbr{\cW_{D},x}}_{y} - \max_{i \neq y} \rbr{f\rbr{\cW_{D},x}}_i$. For a given real value for the margin $\gamma >0$, the class of ramp risk functions $\cG_{\gamma} \rbr{\cF_{D,\Jac} }$ with the margin $\gamma$ and $\frac{1}{\gamma}$-Lipschitz continuous function $g_{\gamma}$ is defined as
\begin{align}
	&\cG_{\gamma} \rbr{\cF_{D,\Jac} } = \cbr{g_{\gamma}\rbr{f\rbr{\cW_{D},x},y} | f_{D}\in \cF_{D,\Jac}}, \label{eqn_dnn:ramp} \\
	&g_{\gamma}\rbr{f\rbr{\cW_{D},x},y} = \left\{ 
		\begin{array}{ll}
		0,& \nu_{\cW_{D}}^{x,y} > \gamma \\
		1 - \frac{\nu_{\cW_{D}}^{x,y}}{\gamma},& \nu_{\cW_{D}}^{x,y} \in \sbr{0,\gamma} \\
		1,& \nu_{\cW_{D}}^{x,y} < 0,
		\end{array}
		\right. \nonumber
\end{align}
For convenience, we denote $g_{\gamma}\rbr{f\rbr{\cW_{D},x},y}$ as $g_{\gamma}\rbr{f\rbr{\cW_{D},x}}$ (or $g_{\gamma}$) in the rest of the paper.

Then the generalization error bound \cite{bartlett2017spectrally} \removed{(Lemma 3.1)} states the  following. Given any real $\delta \in \rbr{0,1}$ and $g_{\gamma}$, with probability at least $1-\delta$, we have that for any $f\rbr{\cdot,\cdot} \in \cF_{D,\Jac}$, the generalization error for classification is upper bounded with respect to (w.r.t.) the ERC satisfies
\begin{align}
\mathop{\PP} \sbr{  \argmax_{j} \rbr{f\rbr{\cW_{D},x}}_j \neq y} \leq \frac{1}{m} \sum_{i=1}^{m} g_{\gamma}\rbr{f\rbr{\cW_{D},x_i}} + 2 \cR_m\rbr{\cG_{\gamma} \rbr{\cF_{D,\Jac} }} +  3 \sqrt{\frac{\log \rbr{\frac{2}{\delta}}}{2m} }. \label{eqn_dnn:gen_err_cf}
\end{align}

The right hand side (R.H.S.) of \eqref{eqn_dnn:gen_err_cf} is viewed as a guaranteed error bound for the gap between the testing and the empirical training performance. Since the ERC is generally the dominating term in \eqref{eqn_dnn:gen_err_cf}, a small $\cR_m$ is desired for DNNs given the loss function $g_{\gamma}$. Analogous results hold for regression tasks; see e.g., \cite{kearns1994introduction,mohri2012foundations} for details. 

\section{Generalization Error Bound for DNNs}\label{sec:dnn}


\subsection{A Tighter ERC Bound for DNNs}

We first provide the ERC bound for the class of DNNs defined in \eqref{eqn_dnn:nn_Dl} and the ramp loss functions in the following theorem. The proof is provided in Appendix~\ref{pf:thm:tighter_upperbd}.
\begin{theorem}\label{thm:tighter_upperbd}
Let $g_{\gamma}$ be a $\frac{1}{\gamma}$-Lipschitz loss function and ${\cF}_{D,\Jac}$ be the class of DNNs defined in \eqref{eqn_dnn:nn_Dl}, $p_d = p$, $r_d = r$ for all $d \in [D]$, $B_{1:D}^{\Jac} = \max_{x \in \cX_{m}} B_{1:D}^{\Jac,x}$, $B^{\Jac}_{\backslash d} = \max_{d \in [D], x \in \cX_{m}} B^{\Jac, x}_{1:(d-1)} B^{\Jac, x}_{(d+1):D}$, and $C^{\Net} = \frac{B^{\Jac}_{\backslash d} \cdot R \sqrt{Dm/r} \cdot \max_{d} B_{d,2}/\gamma}{\sup_{f \in \cF_{D,\nbr{\cdot}_2}, x \in \cX_{m} } g_{\gamma} \rbr{f\rbr{\cW_{D},x}}}$. Then we have
\begin{align*}
\cR_m \rbr{\cG_{\gamma} \rbr{{\cF}_{D,\Jac}} } = \cO \rbr{ \frac{R \cdot B_{1:D}^{\Jac} \sqrt{D pr \log C^{\Net} } }{\gamma \sqrt{m}} }. 
\end{align*}
\end{theorem}

\begin{remark}\label{rk:jac}
Note that $C^{\Net}$ depends on the norm of Jacobian, which is significantly smaller than the product of matrix norms that is exponential on $D$ in general. For example, when we obtain the network from stochastic gradient descent using randomly initialized weights, then $B^{\Jac} \ll \prod_{d} B_{d,2}$. Empirical distributions of $B^{\Jac}$ and $\prod_{d} B_{d,2}$ are provided in Figure~\ref{fig:compare} (c) -- (e), where $B^{\Jac}$ has a dependence slower than some low degree poly(depth), rather than exponential on the depth as in $\prod_{d} B_{d,2}$. Thus, $\log C^{\Net}$ can be considered as a constant almost independent of $D$ in practice. Even in the worst case that $B^{\Jac} \approx \prod_{d} B_{d,2}$ (this almost never happens in practice), our bound is still tighter than existing spectral norm based bounds \cite{bartlett2017spectrally,neyshabur2017pac} by an order of $\sqrt{D}$. Also note that $C^{\Net}$ is a quantity (including $B^{\Jac}_{\backslash d}$) only depending on the training dataset due to the ERC bound. 
\end{remark}

For convenience, we treat $R$ as a constant. From Remark~\ref{rk:jac}, we also treat $\log C^{\Net}$ as a constant w.l.o.g. We achieve \[\tilde{\cO} ( {B_{1:D}^{\Jac} \sqrt{D pr/m}}/\gamma )\] in Theorem~\ref{thm:tighter_upperbd}, which is significantly tighter than existing results based on the network sizes and norms of weight matrices, as shown in Table~\ref{table:compare}. In particular, \cite{neyshabur2015norm} show an exponential dependence on $D$, i.e., \[\cO (2^D \Pi_{d=1}^{D} B_{d,\Fr} / (\gamma\sqrt{m})) \], which can be significantly larger than ours. \cite{bartlett2017spectrally,neyshabur2017pac} demonstrate polynomial dependence on sizes and the spectral norm of weights, i.e., \[\tilde{\cO} ( {\Pi_{d=1}^D B_{d,2} \sqrt{D^3 pr/m}}/\gamma )\]. Our result in Theorem~\ref{thm:tighter_upperbd} is tighter by an order of $D$, which is significant in practice. Moreover, \cite{golowich2017size} demonstrate a bound w.r.t the Frobenius norm as \[\tilde{\cO} \big({\Pi_{d=1}^D B_{d,\Fr}} \min \big\{ \sqrt{\frac{D}{m}}, {m^{-\frac{1}{4}} \log^{\frac{3}{4}}\rbr{m} \sqrt{\log \rbr{ C^{} } }  \big\} }/\gamma \big)\], where \[C^{} = \frac{R \cdot \Pi_{d=1}^D B_{d,\Fr}}{\sup_{x \in \cX_{m}} \nbr{f\rbr{\cW_{D},x}}_2 }\]. This has a tighter dependence on network sizes. Nevertheless, $\nbr{W_d}_{\Fr}$ is generally $\sqrt{r}$ times larger than $\nbr{W_d}_{2}$. Thus, $\Pi_{d=1}^D B_{\Fr,2}$ is \[r^{D/2}\] times larger than $\Pi_{d=1}^D B_{2,2}$. 
Moreover, $\log (C^{})$ is linear on $D$ except that the stable ranks $\nbr{W_d}_{\Fr}/\nbr{W_d}_{2}$ across all layers are close to 1 (rather than $\log C^{\Net}$ being almost independent on $D$ in Theorem~\ref{thm:tighter_upperbd}). In addition, it has \[m^{-\frac{1}{4}}\] dependence rather than \[m^{-\frac{1}{2}}\] except when $D = \cO\rbr{\sqrt{m}}$. Numerical comparison is provided in Figure~\ref{fig:compare} (a), where our bound is orders of magnitude better than the others. Note that our bound is based on a novel characterization of Lipschitz properties of DNNs, which may be of independent interest. We refer to Appendix~\ref{pf:thm:tighter_upperbd} for details.

We also remark that when achieving the same order of generalization errors, we allow the choices of larger dimensions ($D,p$) and norms of weight matrices, which lead to stronger expressive power for DNNs. For example, even in the worst case that $B^{\Jac} \approx \prod_{d} B_{d,2}$, when achieving the same bound with $\nbr{W_d}_{2}=1$ in spectral norm based results (e.g. in ours) and $\nbr{W_d}_{\Fr}=1$ in Frobenius norm based results (e.g., in \cite{golowich2017size}), they only have $\nbr{W_d}_{2}=\cO(1/\sqrt{r})$ in Frobenius norm based results. The later results in a much smaller space for eligible weight matrices as $r$ is of order $p$ in general (i.e., $r =  \delta p$ for some constant $\delta \in (0,1)$), which leads to weaker expressive ability of DNNs. We also demonstrate numerically in Figure~\ref{fig:compare} (b) that when norms of weight matrices are constrained to be very small, both training and testing performance degrade significantly. 

\subsection{ERC Bound for Bounded Loss}\label{sec:indep}

When, in addition, the loss function is bounded, we have that the ERC bound can be free of the Jacobian term, as in the following corollary. The proof is provided in Appendix~\ref{pf:cor:norm_ind}.
\begin{corollary}\label{cor:norm_ind}
In addition to the conditions in Theorem~\ref{thm:tighter_upperbd}, suppose we further let $g_{\gamma}$ be bounded, i.e., $\abr{g_{\gamma}} \leq b$. Then the ERC satisfies
\begin{align}
\cR_m \rbr{\cG_{\gamma} \rbr{{\cF}_{D,\Jac}} } = \cO\rbr{ \min\cbr{ \frac{R \cdot B_{1:D}^{\Jac} }{\gamma}, b } \cdot \sqrt{\frac{D pr \log C^{\Net}}{m} } }. \label{eqn:norm_ind}
\end{align}
\end{corollary}
The boundedness of $g_{\gamma}$ holds for certain loss functions, e.g., the ramp risk defined in \eqref{eqn_dnn:ramp} and cross entropy loss. When $b$ is constant (e.g., $b=1$ for the ramp risk) and $R B_{1:D}^{\Jac} \gg \gamma$, we have that the ERC reduces to $\tilde{\cO} (\sqrt{{D pr}/{m}})$. This is close to the VC dimension of DNNs, which can be significantly tighter than existing norm based bounds in general. Moreover, our bound \eqref{eqn:norm_ind} is also tighter than recent results that are free of the dependence on weight norms \cite{zhou2018understanding,arora2018stronger}. For example, \cite{zhou2018understanding} show that the generalization bound for CNNs is $\tilde{\cO}\big({D\sqrt{pr^2/m }}\big)$, which results in a bound larger than our \eqref{eqn:norm_ind} by $\cO(\sqrt{D r})$. \cite{arora2018stronger} derive a bound for a compressed network in terms of some error-resilience parameters, which is $\tilde{\cO}(\sqrt{D^3 p^2/m})$ since the cushion parameter therein is of the order $\mu=\cO(1/\sqrt{p})$. Similar norm free results hold for the architectures discussed in Section~\ref{sec:structure} using argument for Corollary~\ref{cor:norm_ind}, which we skip due to space limit. 

\section{Exploring Network Structures}\label{sec:structure}

The generic result in Section~\ref{sec:dnn} can be further highlighted explicitly using specific structures of the networks. In this section, we consider two popular architectures of DNNs, namely convolutional neural networks (CNNs) \cite{krizhevsky2012imagenet} and residual networks (ResNets) \cite{he2016deep}, and provide sharp characterization of the corresponding generalization bounds. In particular, we consider orthogonal filters and normalized weight matrices, which have shown good performance in both optimization and generalization \cite{mishkin2015all,huang2017orthogonal}. Such constraints can be enforced using regularizations on filters and weight matrices, which is very efficient to implement in practice. This is also closely related with normalization approaches, e.g., batch normalization \cite{ioffe2015batch} and layer normalization \cite{ba2016layer}, which have achieved tremendous empirical success. 

\subsection{CNNs with Orthogonal Filters}\label{sec:cnn}

CNNs are one of the most powerful architectures in deep learning, especially in tasks related with images and videos \cite{goodfellow2016deep}. We consider a tight characterization of the generalization bound for CNNs by generating the weight matrices using unit norm orthogonal filters, which has shown great empirical performance \cite{huang2017orthogonal,xie2017all}. Specifically, we generate the weight matrices using a circulant approach, as follows. For the convolutional operation at the $d$-th layer, we have $n_{d}$ channels of convolution filters, each of which is generated from a $k_{d}$-dimensional feature using a stride side $s_{d}$. Suppose that $s_{d}$ divides both $k_{d}$ and $p_{d-1}$, i.e., $\frac{k_{d-1}}{s_{d}}$ and $\frac{p_{d-1}}{s_{d}}$ are integers, then we have $p_{d} = \frac{n_{d} \cdot p_{d-1}}{s_{d}}$. This is equivalent to fixing the weight matrix at the $d$-th layer to be generated as in \eqref{eqn_dnn:cnn_weight}, where for all $j \in [n_{d}]$, each $W_{d}^{(j)} \in \RR^{\frac{p_{d-1}}{s_{d}} \times p_{d-1}}$ is formed in a circulant-like way using a vector $w^{(d,j)} \in \RR^{k_{d}}$ with unit norms for all $j$ as 
\begin{align}
&W_{d} = \sbr{W_{d}^{(1)\top} \cdots~W_{d}^{(n_{d})\top}}^\top \in \RR^{p_{d} \times p_{d-1}},  \label{eqn_dnn:cnn_weight}\\
&W_{d}^{(j)} = \sbr{\begin{array}{c}
w^{(d,j)} ~\underbrace{0\cdot\cdots\cdots\cdots\cdots\cdots\cdots\cdots 0}_{\in \RR^{p_{d-1}-k_{d}}} \\
\underbrace{0\cdots 0}_{\in \RR^{s_d}} ~ w^{(d,j)} \underbrace{0\cdot\cdots\cdots\cdots\cdots\cdots 0}_{\in \RR^{p_{d-1}-k_{d}-s_{d}} } \\
\vdots \\
w^{(d,j)}_{(s_{d}+1):k_{d}} \underbrace{0\cdots\cdots\cdots\cdots 0}_{\in \RR^{p_{d-1}-k_{d}}} ~w^{(d,j)}_{1:s_{d}}
\end{array} }. \label{eqn_dnn:cnn_weight2}
\end{align}

When the stride size $s_{d}=1$, $W_{d}^{(j)}$ corresponds to a standard circulant matrix \cite{davis2012circulant}. The following lemma establishes that when $\cbr{ w^{(d,j)} }_{j=1}^{n_{d}}$ are orthogonal vectors with unit Euclidean norms, the generalization bound only depend on $s_{d}$ and $k_{d}$ that are independent of the width $p_{d}$. The proof is provided in Appendix~\ref{pf:cor:cnn_bd}.

\begin{corollary}\label{cor:cnn_bd}
Let $g_{\gamma}$ be a $\frac{1}{\gamma}$-Lipschitz and bounded loss function, i.e., $\abr{g_{\gamma}} \leq b$, and $\cF_{D,\Jac}$ be the class of CNNs defined in \eqref{eqn_dnn:nn_Dl}. Suppose the weight matrices in CNNs are formed as in \eqref{eqn_dnn:cnn_weight} and \eqref{eqn_dnn:cnn_weight2} with $s_{d}=s$, $k_d = k$, and $s$ divides both $k$ and $p_{d}$ for all $d \in [D]$, where $\cbr{ w^{(d,j)} }_{j=1}^{n_{d}}$ satisfies $w^{(j)\top} w^{(i)} = 0$ for all $i,j \in [n_{d}]$ and $i \neq j$ with $\nbr{w^{(d,j)}}_2 = 1$ for all $j \leq n_{d}$. Denote $C^{\Net} = \frac{B^{\Jac}_{\backslash d,2} \cdot R \sqrt{Dm/s}/\gamma}{\mathop{\sup}_{f \in \cF_{D,\Jac}, x \in \cX_{m} } g_{\gamma} \rbr{f\rbr{\cW_{D},x}}}$. Then the ERC satisfies
\begin{align*}
\cR_m \rbr{\cG_{\gamma} \rbr{{\cF}_{D,\Jac}} } = \cO\rbr{ \min\cbr{ \frac{R \rbr{{k}/{s}}^{D/2} }{\gamma}, b } \cdot \sqrt{\frac{k \sum_{d=1}^{D} n_{d} \cdot \log C^{\Net}}{m} } }.
\end{align*}
\end{corollary}
Since $n_{d} \leq k$ in our setting, the ERC for CNNs is proportional to $\sqrt{D k^2}$ instead of $\sqrt{Dpr}$. For the orthogonal filtered considered in Corollary~\ref{cor:cnn_bd}, we have $\nbr{W_d}_{\Fr} = \sqrt{p_d}$ and $\nbr{W_d}_{2,1} = p_d$, which lead to the bounds of CNNs in existing results in Table~\ref{table:compare_cnn}. In practice, one usually has $k_d \ll p_d$, which exhibit a significant improvement over existing results, i.e., $\sqrt{D k^2} \ll \sqrt{D^3 p^2}$. Even without the orthogonal constraint on filters, the rank $r$ in CNNs is usually of the same order with width $p$, which also makes the existing bound undesirable. On the other hand, it is widely used in practice that $k_{d} = \mu s_{d}$ for some small constant $\mu\geq 1$ in CNNs, then we have $({k_d}/{s_{d}})^{D/2} \ll p^{D/2}$ resulted from $\prod_{d} B_{d,\Fr}$.

\begin{table}[t]
\begin{center}
\caption{Comparison with existing norm based capacity bounds of CNNs. We suppose $R$ and ${\gamma}$ are generic constants for ease of illustration. The results of CNNs in existing works are obtained by substituting the corresponding norms of the weight matrices generated by orthogonal filters, i.e., $\nbr{W_{d}}_{2} = \sqrt{{k}/{s}}$, $\nbr{W_{d}}_{\Fr} = \sqrt{p}$, and $\nbr{W_{d}}_{2,1} = p$.}
{
	\renewcommand{\arraystretch}{1.5}
	\begin{tabular}{c|c}
		\Xhline{1 pt}
		Capacity Bound & CNNs  \\
		\hline
		\hspace{-0.09in}\cite{neyshabur2015norm}\hspace{-0.09in} & $\cO \Big( \frac{ 2^D \cdot  p_{}^{\frac{D}{2}} }{\sqrt{m}} \Big)$ \\
		\hline
		\cite{bartlett2017spectrally} & $\widetilde{\cO} \bigg( \frac{  \rbr{\frac{k}{s}}^{\frac{D-1}{2}} \cdot \sqrt{D^3 p^2} }{\sqrt{m}} \bigg)$ \\
		\hline
		\hspace{-0.09in}\cite{neyshabur2017pac}\hspace{-0.09in} & $\widetilde{\cO} \bigg( \frac{  \rbr{\frac{k}{s}}^{\frac{D-1}{2}} \cdot \sqrt{D^3 p^2} }{\sqrt{m}} \bigg)$ \\
		\hline
		\cite{golowich2017size} & \hspace{-0.07in} $\widetilde{\cO}\rbr{ p^{\frac{D}{2}} \min\cbr{\frac{1}{\sqrt[4]{m}}, \sqrt{\frac{D}{m}}} }$ \hspace{-0.11in} \\
		\hline
		Our results & \note{$\tilde{\cO} \bigg( \frac{  \rbr{\frac{k}{s}}^{\frac{D}{2}}\sqrt{D k^2} }{\sqrt{m}} \bigg)$} \\
		\Xhline{1 pt}
	\end{tabular}\label{table:compare_cnn}
}
\end{center}
\end{table}

\begin{remark}\label{rk:2d_input}
Vector input is considered in Corollary~\ref{cor:cnn_bd}. For matrix inputs, e.g., images, similar results hold by considering vectorized input and permuting columns of $W_{d}$. Specifically, suppose $\sqrt{k_{d}}$ and $\sqrt{p_{d-1}}$ are integers for ease of discussion. Consider the input as a $p_{d-1}$ dimensional vector obtained by vectorizing a $\sqrt{p_{d-1}} \times \sqrt{p_{d-1}}$ input matrix. When the 2-dimensional (matrix) convolutional filters are of size $\sqrt{k_{d}} \times \sqrt{k_{d}}$, we form the rows of each $W_{d}^{(j)}$ by concatenating $\sqrt{k_{d}}$ vectors $\{w^{(j,i)}\}_{i=1}^{\sqrt{k_{d}}}$ 
padded with 0's, each of which is a concatenation of one row of the filter of size $\sqrt{k_{d}}$ with some zeros as follow:
\begin{align*}
\underbrace{w^{(j,1)}}_{\in \RR^{\sqrt{k_{d}}}} \underbrace{0 \cdots\cdots 0}_{\in \RR^{\sqrt{\frac{p_{d-1}}{k_{d}}}-\sqrt{k_{d}}}} \cdots\cdots
	\underbrace{w^{(j,\sqrt{k_{d}})}}_{\in \RR^{\sqrt{k_{d}}}} \hspace{-0.0in}\underbrace{0 \cdots\cdots 0}_{\in \RR^{\sqrt{\frac{p_{d-1}}{k_{d}}}-\sqrt{k_{d}}}} 
	\underbrace{0 \cdots\cdots\cdots 0}_{\in \RR^{p_{d-1} - \sqrt{p_{d-1}}}}.
\end{align*}

Correspondingly, the stride size is $\frac{s_{d}^2}{k_{d}}$ on average and we have $\nbr{W_{d}}_2 \leq \frac{k_{d}}{s_{d}}$ if $\nbr{w^{(j,i)}}_2 = 1$ for all $i,j$; see Appendix~\ref{apx:mat_filter} for details. This is equivalent to permuting the columns of $W_{d}$ generated as in \eqref{eqn_dnn:cnn_weight2} by vectorizing the matrix filters in order to validate the convolution of the filters with all patches of the matrix input. 
\end{remark}
\begin{remark}
A more practical scenario for CNNs is when a network has a few fully connected layers after the convolutional layers. Suppose we have $D_C$ convolutional layers and $D_F$ fully connected layers. From the analysis in Corollary~\ref{cor:cnn_bd}, when $s_{d} = k_{d}$ for convolutional layers and $\nbr{W_d}_2 = 1$ for fully connected layers, we have that the overall ERC satisfies \[\tilde{\cO} \Big( \frac{R \cdot \sqrt{D_C k^2 + D_F pr }}{\gamma \sqrt{m}} \Big)\].
\end{remark}

\subsection{ResNets with Structured Weight Matrices}\label{eqn:resnet}

Residual networks (ResNets) \cite{he2016deep} is one of the most powerful architectures that allows training of tremendously deep networks. Given an input $x \in \RR^{p_0}$, the output of a $D$-layer ResNet is defined as $f\rbr{\cV_{D},\cU_{D},x} = f_{V_{D},U_{D}} \rbr{\cdots f_{V_{1},U_{1}}\rbr{x} } \in  \RR^{p_D}$, where $f_{V_{d},U_{d}}\rbr{x} = \sigma \rbr{ V_{d} \cdot \sigma \rbr{ U_{d} x} + x}$. For any two layers $i,j \in [D]$ and input $x$, we denote $J_{i:j}^x$ as the Jacobian from layer $i$ to layer $j$, i.e., $f_{V_{i},U_{j}} \rbr{\cdots f_{V_{i},U_{i}}\rbr{x} } = J_{i:j}^x \cdot x$. 
Then we denote the class of ResNets with bounded weight matrices $\cV_{D} = \cbr{V_d \in \RR^{p_{d} \times q_{d}} }_{d=1}^D$, $\cU_{D} = \cbr{U_d \in \RR^{q_{d} \times p_{d-1}} }_{d=1}^D$ as
\begin{align}
\cF^{\RN}_{D,\Jac} = \big\{f\rbr{\cV_{D},\cU_{D},x} \in \RR^{p_D} ~\big|~ \text{sup}_{\cW_{D}} \nbr{J_{i,j}^x}_2 \leq B^{\Jac, x}_{i:j} \big\}, \label{eqn_dnn:resnet_Dl}
\end{align}
We also denote $\nbr{U_d}_2 \leq B_{U_d,2}$ and $\nbr{V_d}_2 \leq B_{V_d,2}$. We then provide an upper bound of the ERC for ResNets in the following corollary. The proof is provided in Appendix~\ref{pf:cor:upperbd_resnet_indep}.
\begin{corollary}\label{cor:upperbd_resnet_indep}
Let $g_{\gamma}$ be a $\frac{1}{\gamma}$-Lipschitz and bounded loss function, i.e., $\abr{g_{\gamma}} \leq b$, and $\cF^{\RN}_{D,\Jac}$ be the ResNets defined in \eqref{eqn_dnn:resnet_Dl} with $p_d = p$ and $q_d = q$ for all $d \in [D]$, $B_{1:D}^{\Jac} = \max_{x \in \cX_{m}} B_{1:D}^{\Jac,x}$, $B^{\Jac}_{\backslash d} = \max_{d \in [D], x \in \cX_{m}}$ $B^{\Jac, x}_{1:(d-1)} B^{\Jac, x}_{(d+1):D}$, and $C^{\Net} =  \frac{B^{\Jac}_{\backslash d} \max_{d} \rbr{B_{V_d,2} + B_{U_d,2}} R\sqrt{m/q} /\gamma}{\sup_{f \in \cF_{D,\Jac}, x \in \cX_{m} } g_{\gamma} \rbr{f\rbr{\cV_{D},\cV_{D},x}} }$. Then the ERC satisfies
\begin{align*}
\cR_m \rbr{\cG_{\gamma} \rbr{{\cF}^{\RN}_{D,\Jac}} } = \cO\rbr{ \min\cbr{ \frac{R \cdot B_{1:D}^{\Jac} }{\gamma}, b } \cdot \sqrt{\frac{D pq \cdot \log C^{\Net}}{m} } }.
\end{align*}
\end{corollary}

Compared with the $D$-layer networks without shortcuts in \eqref{eqn_dnn:nn_Dl}, ResNets have a stronger dependence on the input due to the skip-connection structure. In practice, the norm of the weight matrices in ResNets are usually significantly smaller than regular DNNs \cite{he2016deep}, which helps maintain relatively small $B_{1:D}^{\Jac}$ values. 

%

\subsection{Extension to Width-Change Operations}\label{sec:operation}

Changing the width for certain layers is a widely used operation, e.g., for CNNs and ResNets, which can be viewed as a linear transformation in many cases. In specific, we use $T_{d} \in \RR^{p_{d+1} \times p_{d}}$ to denote the operation to change the dimension from the $d$-th layer to the $(d+1)$-th layer as $f_{W_{d+1}}\rbr{x} =  \sigma \rbr{ W_{d+1} T_{d} x }$. 
Denote the set of layers with width changes by $\cI_T \subseteq [D]$. Combining with Theorem~\ref{thm:tighter_upperbd}, we have that the ERC satisfies $\cR_m \rbr{\cG_{\gamma} \rbr{\cF_{D,\Jac}} } = \tilde{\cO} \rbr{ \frac{R \cdot B_{1:D}^{\Jac} \cdot \note{\Pi_{t \in \cI_T} \nbr{T_{t}}_2 } \cdot \sqrt{Dpr}}{ \gamma \sqrt{m}} }$. Next, we illustrate several popular examples to show that $\Pi_{t \in \cI_T} \nbr{T_{t}}_2$ is a size independent constant. We refer to \cite{goodfellow2016deep} for more operations of changing the width.

\vspace{0.1in}

\noindent{\bf Width Expansion.} Two popular types of width expansion are padding and $1 \times 1$ convolution. Suppose $p_{d+1} = s \cdot p_{d}$ for some positive integer $s \geq 1$. Taking padding with 0 as an example, we have $\rbr{T_{d}}_{ij} = 1$ if $i=j s$, and $\rbr{T_{d}}_{ij} = 0$ otherwise for $T_{d} \in \RR^{sp_{d} \times p_{d}}$. This implies that $\nbr{T_{d}}_2 = 1$. 

For $1 \times 1$ convolution, suppose that the convolution features are $\cbr{c_1,\ldots,c_s}$. Then we expand width by performing entry-wise product using $s$ features respectively. This is equivalent to setting $T_{d} \in \RR^{sp_{d} \times p_{d}}$ with $\rbr{T_{d}}_{ij} = c_k$ if $i=j + (k-1) s$ for $k \in [s]$, and $\rbr{T_{d}}_{ij} = 0$ otherwise. It implies that $ \nbr{T_{d}}_2 = \sqrt{\sum_{i=1}^{s} c_i^2} \leq 1$ when $\sum_{i=1}^{s} c_i^2 \leq 1$.

\vspace{0.1in}

\noindent{\bf Width Reduction.} Two popular types of width reduction are average pooling and max pooling. Suppose $p_{d+1} = \frac{p_{d}}{s}$ is an integer. For average pooling, we pool each nonoverlapping $s$ features into one feature. This implies $T_{d} \in \RR^{\frac{p_{d}}{s} \times p_{d}}$ with $\rbr{T_{d}}_{ij} ={1}/{s}$ if $j=(i-1) s + k$ for $k \in [s]$, and $\rbr{T_{d}}_{ij} = 0$ otherwise. Then we have $\nbr{T_{d}}_2 = \sqrt{{1}/{s}} $. 

For max pooling, we choose the largest entry in each nonoverlapping feature segment of length $s$. Denote $I_s = \cbr{(i-1)\times s +1, \ldots,i \cdot s}$. This implies $T_{d} \in \RR^{\frac{p_{d}}{s} \times p_{d}}$ with $\rbr{T_{d}}_{ij} = 1$ if $ | (x^{\cbr{d}})_j | \geq | (x^{\cbr{d}})_k|~\forall ~k \in I_s,k \neq j $, and $\rbr{T_{d}}_{ij} = 0$ otherwise. This implies that $\nbr{T_{d}}_2 = 1$. For pooling with overlapping features, similar results hold.

\section{Numerical Evaluation}\label{sec:exp_all}


To better illustrate the difference between our result and existing ones, we demonstrate numerical results in Figure~\ref{fig:compare} using real data. In specific, we train a simplified VGG19-net \cite{simonyan2014very} using $3 \times 3$ convolution filters (with unit norm constraints) on the CIFAR-10 dataset \cite{krizhevsky2009learning}. 

$\bullet$ \noindent{\bf Comparison of Bounds.} We first compare with the capacity terms in \cite{bartlett2017spectrally} (Bound1), \cite{neyshabur2017pac} (Bound2), and \cite{golowich2017size} (Bound3) by ignoring the common factor $\frac{R}{\gamma\sqrt{m}}$ as follows:

\begin{itemize}
\item Ours: $\Pi_{d=1}^{D} B_{d,2} \sqrt{k \sum_{d=1}^{D} n_{d}}$;
\item Bound1: $\Pi_{d=1}^{D} B_{d,2} \rbr{\sum_{d=1}^{D} \frac{B_{d,2\rightarrow 1}^{{2}/{3}}}{B_{d,2}^{{2}/{3}}} }^{3/2}$;
\item Bound2: $\Pi_{d=1}^{D} B_{d,2} \sqrt{D^2 p \sum_{d=1}^{D} \frac{p_d B_{d,\Fr}^{2}}{B_{d,2}^{2} } }$; 
\item Bound3: $\Pi_{d=1}^{D} B_{d,\Fr} \sqrt{D}$.
\end{itemize}

Note that we use the upper bound $\Pi_{d=1}^{D} B_{d,2}$ rather than $B_{1:D}^{\Jac}$ to make the comparison of dimension dependence more explicit. Further comparison of $\Pi_{d=1}^{D} B_{d,2}$ and $B_{1:D}^{\Jac}$ are in the next experiment. Since we may have more filters $n_{d}$ than their dimension $k$, we do not assume orthogonality here. Thus we simply use the upper bounds of norms $B_{d}$ rather than the form as in Table~\ref{table:compare_cnn}. Following the analysis of Theorem~\ref{thm:tighter_upperbd}, we have \[\sqrt{k \sum_{d=1}^{D} n_{d}}\] dependence rather than \[\sqrt{D pr}\] as \[k \sum_{d=1}^{D} n_{d}\] is the total number free parameter for CNNs, where $n_{d}$ is the number of filters at $d$-th layer. Also note that we ignore the logarithms factors in all bounds for simplicity and their empirical values are small constants compared with the the dominating terms.

\begin{figure}[t!]
\begin{center}
\begin{tabular}{ccc}
\includegraphics[width=0.28\textwidth]{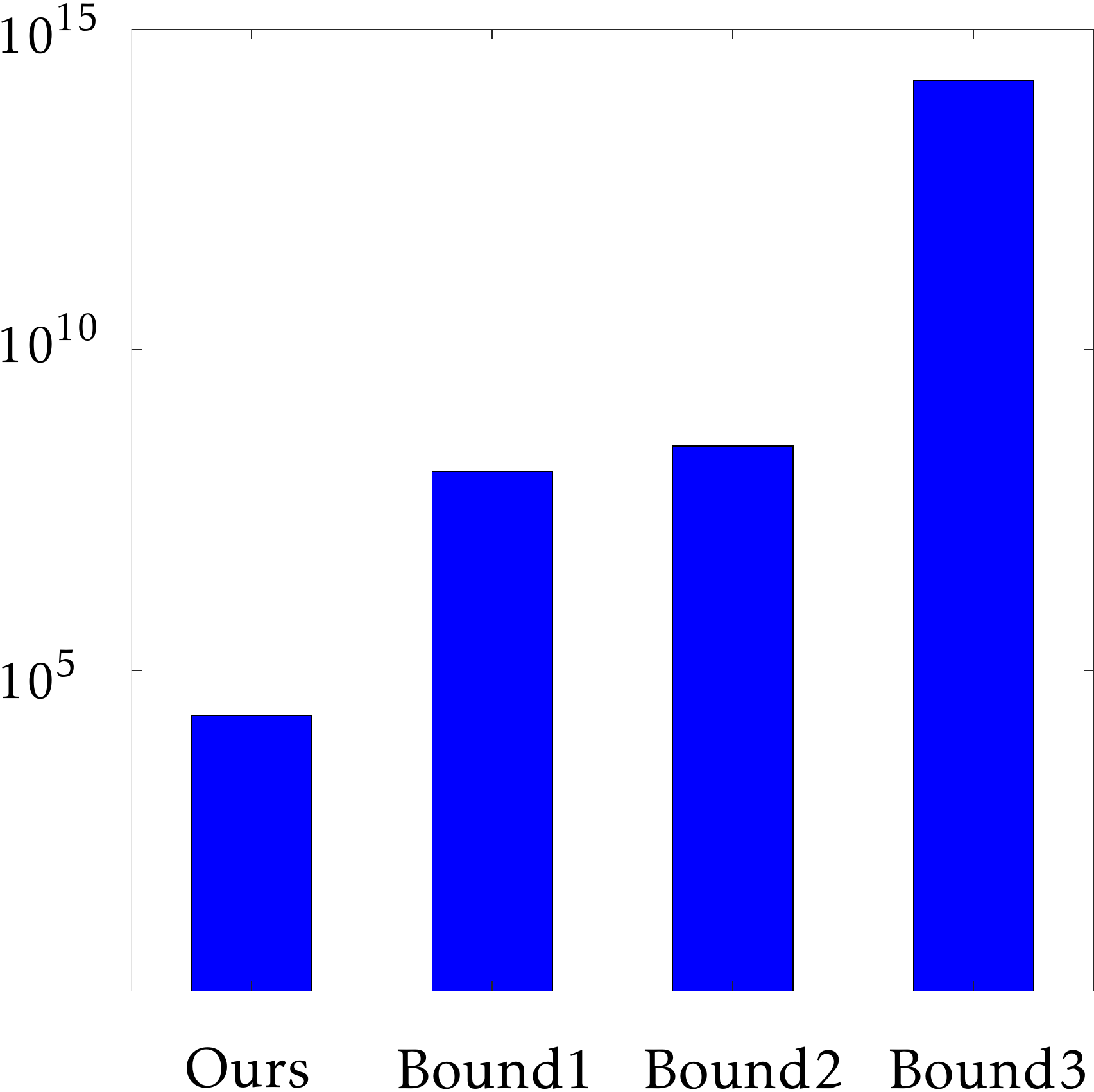} & \hspace{0.08in}\includegraphics[width=0.258\textwidth]{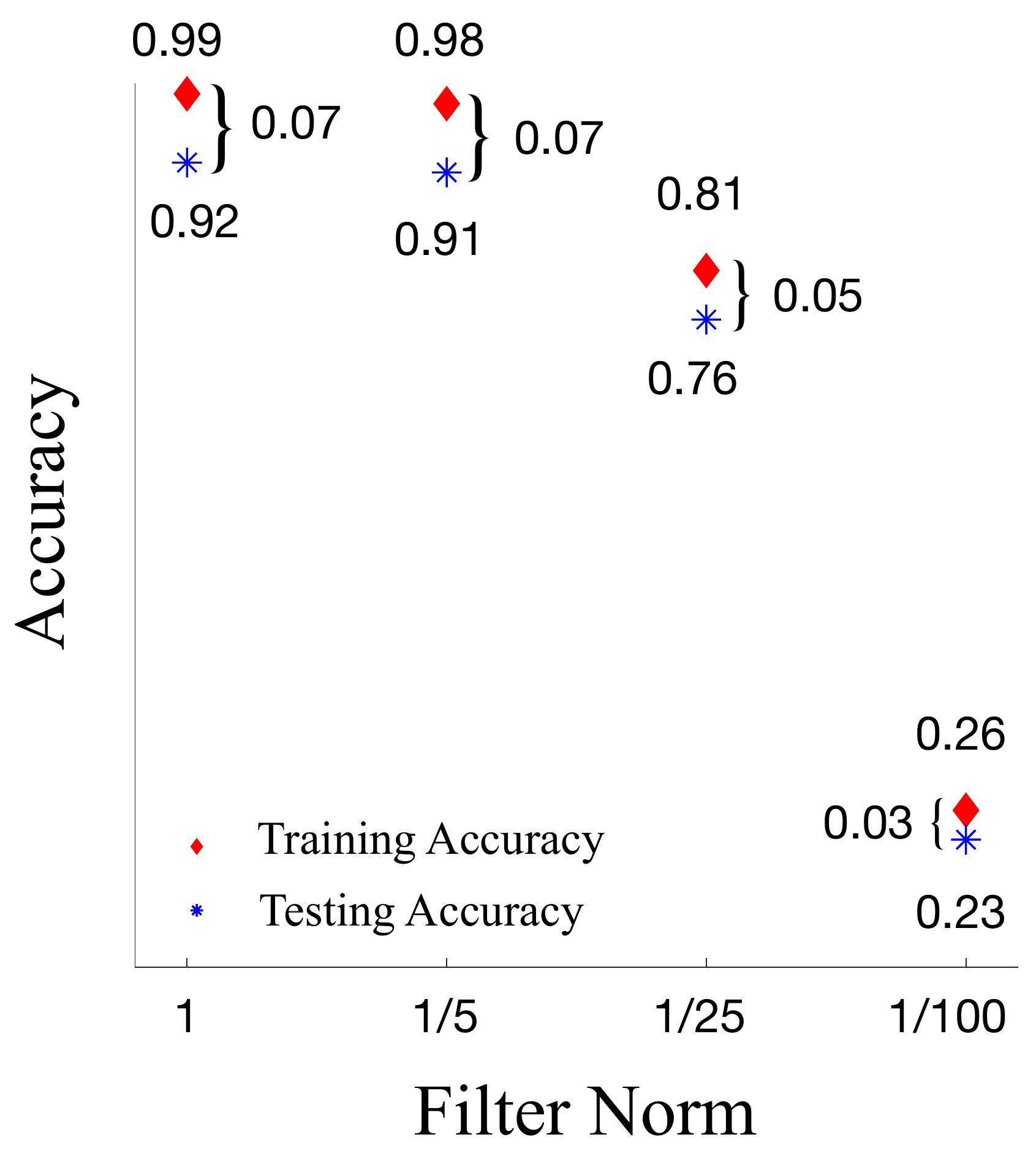} &
\hspace{0.08in}\includegraphics[width=0.32\textwidth]{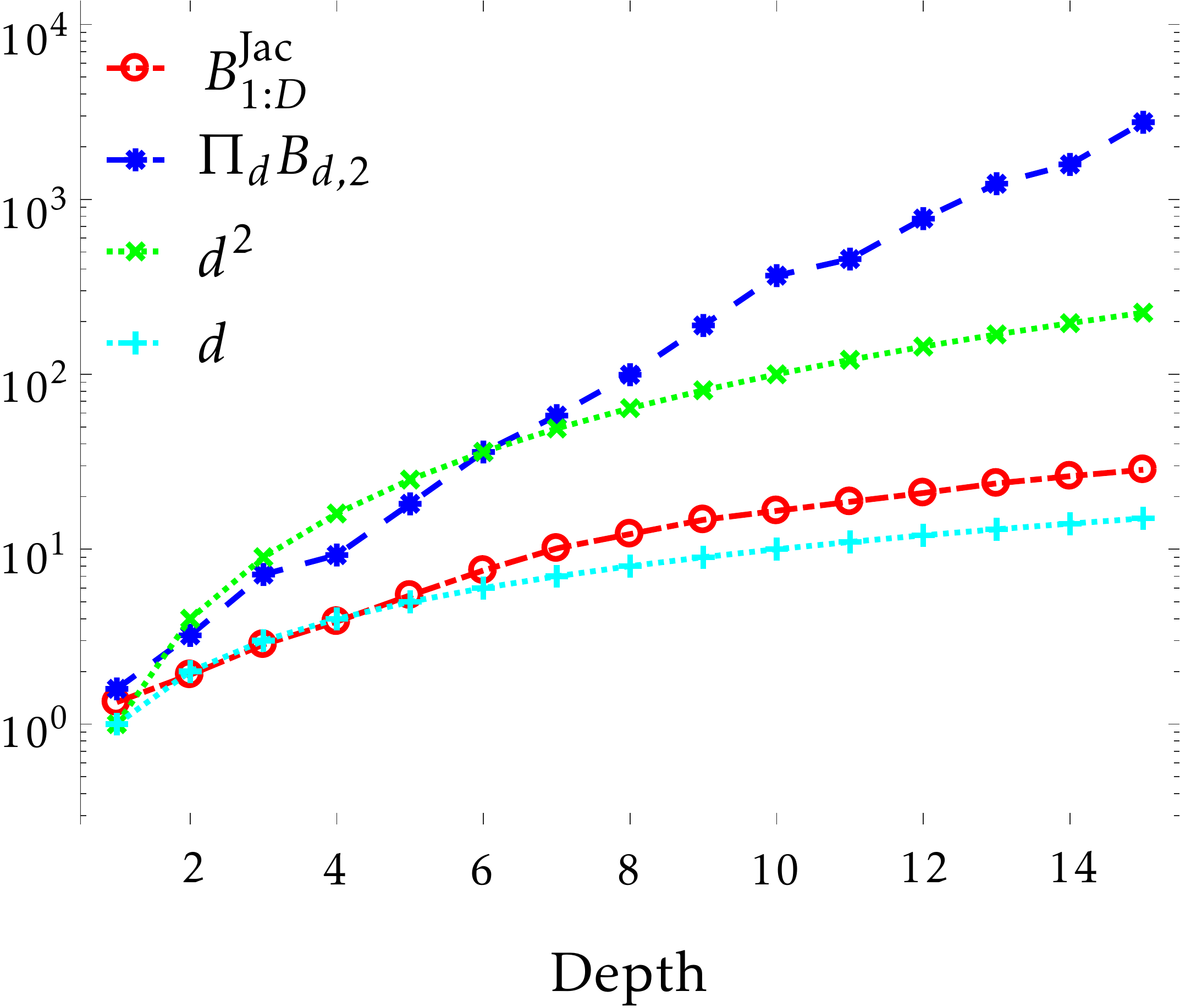} \\
\footnotesize (a) & \footnotesize (b) & \footnotesize (c)
\end{tabular}
\end{center}
\caption{Panel (a) shows comparison results for the same VGG19 network trained on CIFAR10 using unit norm filters. The vertical axis the corresponding bounds in the logarithmic scale. Panel (b) shows the training accuracy (red diamond), testing accuracy (blue cross), and the empirical generalization error using different scales of the filters listed on the horizontal axes.  (c) Comparison results for the dependence of $B^{\Jac}_{1:D}$ and $\prod_{d} B_{d,2}$ on depth, where the horizontal axes is the depth and the vertical axes is the values of corresponding quantities in the logarithmic scale.}\label{fig:compare}
\end{figure}

For the same network and corresponding weight matrices, we see from Figure~\ref{fig:compare} (a) that our result ($10^4 \sim 10^5$) is significantly smaller than \cite{bartlett2017spectrally,neyshabur2017pac} ($10^8 \sim 10^9$) and \cite{golowich2017size}  ($10^{14} \sim 10^{15}$), even we use $\Pi_{d=1}^{D} B_{d,2}$ instead of $B_{1:D}^{\Jac}$. As we have discussed, our bound benefits from tighter dependence on the dimensions. Note that \[k \sum_{d=1}^{D} n_{d}\] is approximately of order \[Dk^2\], which is significantly smaller than \[ \big(\sum_{d=1}^{D} {B_{d,2\rightarrow 1}^{{2}/{3}}}/{B_{d,2}^{{2}/{3}}} \big)^3 \] in \cite{bartlett2017spectrally} and \[ D^2 p \sum_{d=1}^{D} {p_d B_{d,\Fr}^{2}}/{B_{d,2}^{2} } \] in \cite{neyshabur2017pac} (both are of order $D^3 pr$). In addition, this verifies that spectral dependence is significantly tighter than Frobenius norm dependence in \cite{golowich2017size}. Further, we show the training accuracy, testing accuracy, and the empirical generalization error using different scales on the norm of the filters in Figure~\ref{fig:compare} (b). We see that the generalization errors decrease when the norm of filters decreases. However, note that when the norms are too small, the accuracies drop significantly due to a potentially much smaller parameter space. Thus, the scales (norms) of the weight matrices should be nether too large (induce large generalization error) nor too small (induce low accuracy) and choosing proper scales is important in practice as existing works have shown. On the other hand, this also support our claim that when the Frobenius norm based bound attains the same order with the spectral norm based bound, the latter can achieve better training/testing performance.

We want to remark that all numerical evaluations are empirical estimation of the generalization bounds, rather than their exact values. This is because all existing bounds requires to take uniform bounds of some quantities on parameters or the supremum value over the entire space, which is empirically not accessible. For example, in the case that when it involve the upper/lower bound of quantities (norm, rank, or other parameters) depending on weight matrices, theoretically we should take the values of their upper/lower bounds (this leads to worse empirical bounds) rather than estimating them from the training process; or in the case that the bounds involve some quantities depending on the supremum over the entire parameter space, numerical evaluations cannot exhaust the entire parameter space to reach the supremum \cite{bartlett2017spectrally,golowich2017size,neyshabur2015norm,neyshabur2017pac,zhou2018understanding,arora2018stronger}. Similarly, we did not calculate the optimal value of $\gamma$ since it is computational expensive, where the optimal $\gamma$ scales with the $\prod_{d} B_d$ and balances the quantities on the R.H.S. of \eqref{eqn_dnn:gen_err_cf}. Our experiments here cannot avoid such restrictions, but the comparison is fair across various bounds as they are obtained from the same training process. 

\begin{figure}[t!]
\begin{center}
\begin{tabular}{cc}
\includegraphics[width=0.29\textwidth]{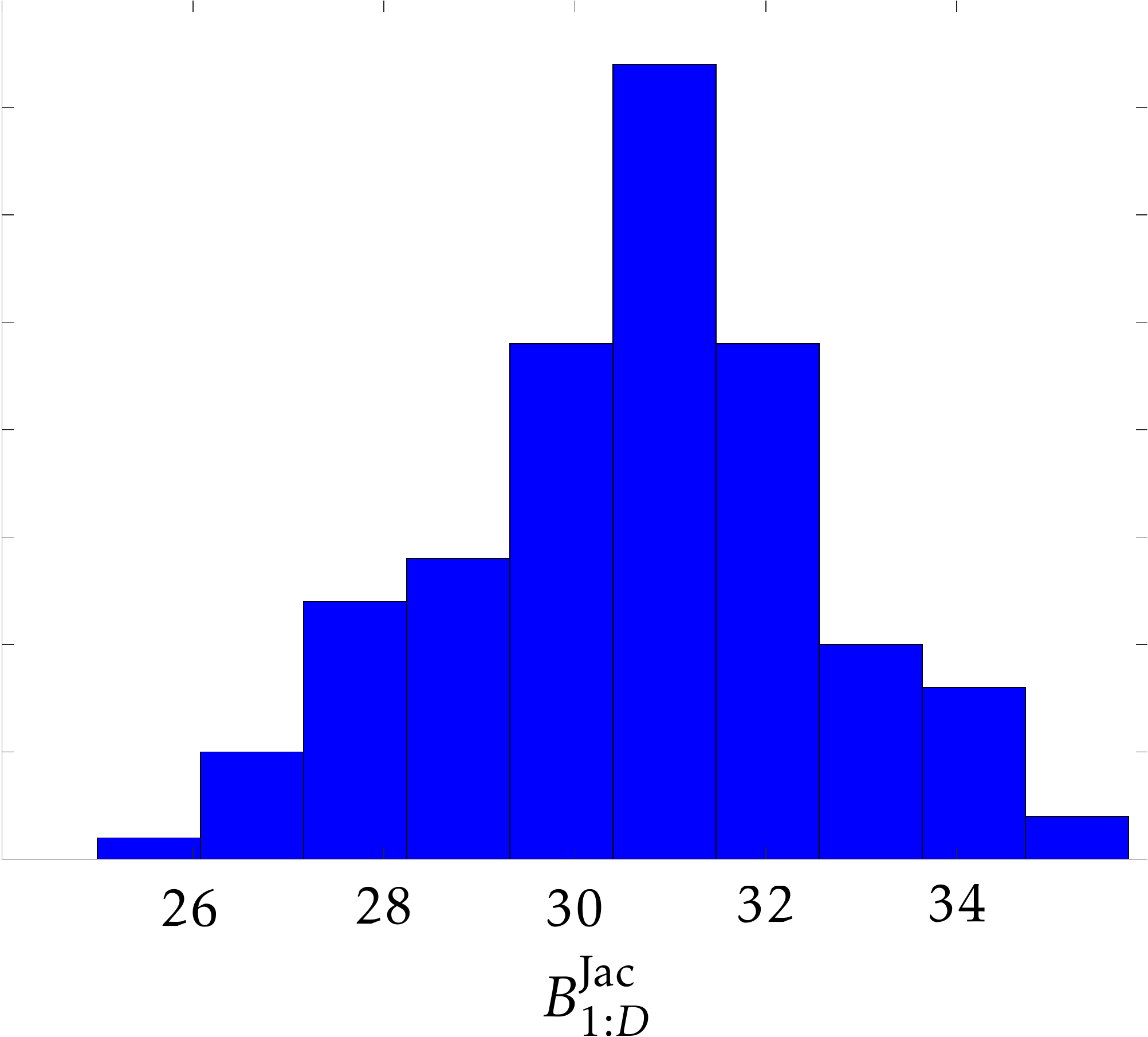} & \hspace{0.18in}\includegraphics[width=0.29\textwidth]{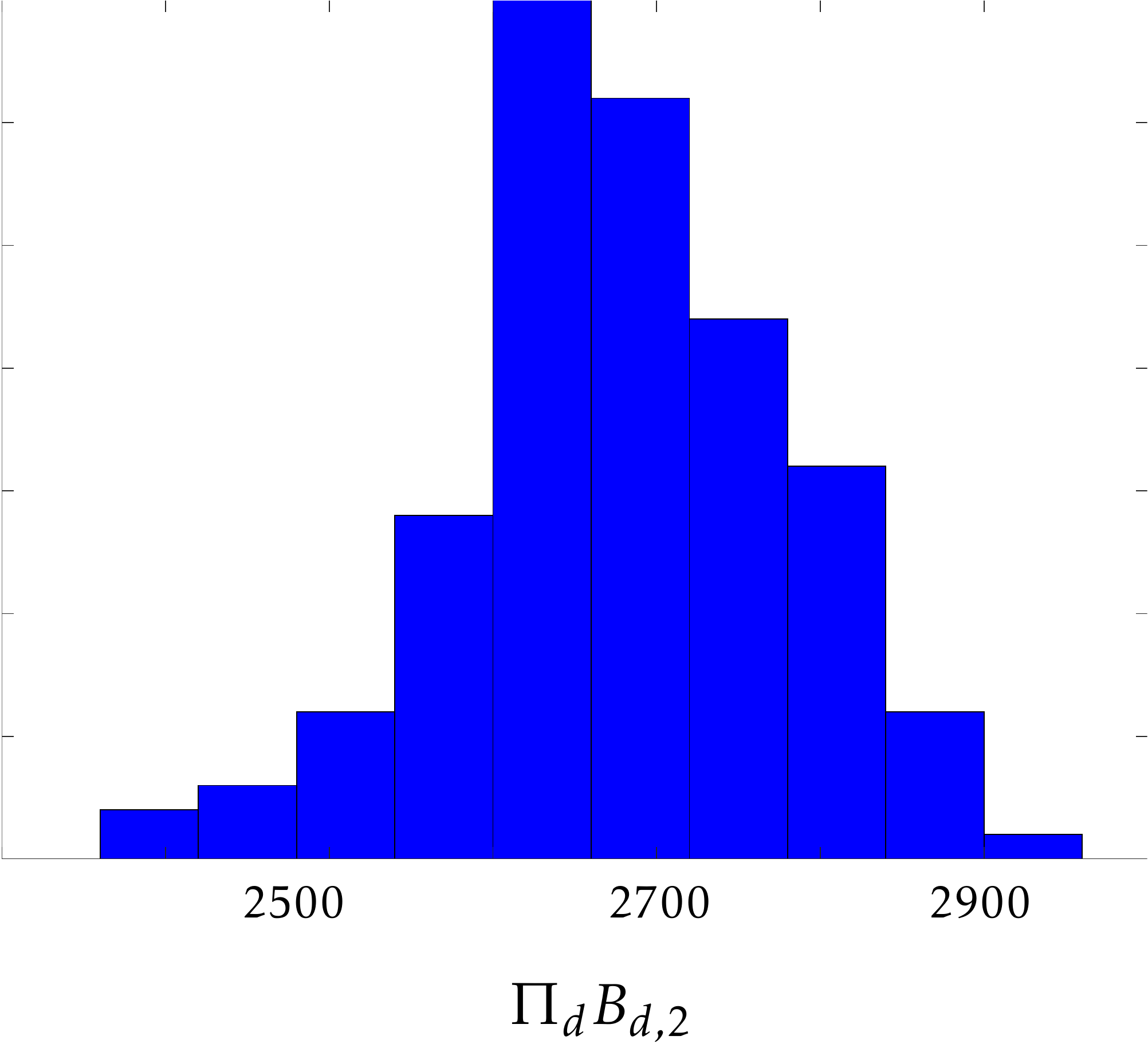} \\
	\footnotesize (a) & \footnotesize (b)
\end{tabular}
\end{center}
\caption{Panel (a) \& (b) Empirical distribution of $B^{\Jac}_{1:D}$ and $\prod_{d} B_{d,2}$ for the same VGG19 network trained on CIFAR10, where the horizontal axes is the empirical values of $B^{\Jac}_{1:D}$ and $\prod_{d} B_{d,2}$, respectively.}\label{fig:compare2}
\end{figure}

$\bullet$ \noindent{\bf Dependence of $B^{\Jac}_{1:D}$ and $\prod_{d} B_{d,2}$ on Depth}. We then provide an empirical evaluation to see how strong the quantities $B^{\Jac}_{1:D}$ and $\prod_{d} B_{d,2}$ depend on the depth. Note that we use $d$ as the variable for depth. Using the same setting as above, we provide the magnitude of $\log B^{\Jac}_{1:D}$ and $\log \prod_{d} B_{d,2}$ in Figure~\ref{fig:compare} (c). We also provide the plots for $\log d$ and $\log d^2$ as reference. We can observe that $\log \prod_{d} B_{d,2}$ has an approximately linear dependence on the depth, which matches with our intuition. In terms of $\log B^{\Jac}_{1:D}$, we can see that it has a significantly slower increasing rate than $\log \prod_{d} B_{d,2}$. Compared with the reference plots $\log d$ and $\log d^2$, we can observe even a slower rate than $\log d^2$. This further indicates that $\log B^{\Jac}_{1:D}$ may has a dependence slower than some low degree of poly$(d)$.


$\bullet$ \noindent{\bf Comparison between $B^{\Jac}_{1:D}$ and $\prod_{d} B_{d,2}$}. We demonstrate the empirical difference between $B^{\Jac}_{1:D}$ and $\prod_{d} B_{d,2}$. Using the same setting of the network and dataset as above, we provide the empirical distribution of $B^{\Jac}_{1:D}$ and $\prod_{d} B_{d,2}$ over the training set using different random initializations of weight matrices, which is provided in Figure~\ref{fig:compare2} (a) and (b). We can observe that the values of $B^{\Jac}_{1:D}$ are approximately 2 orders smaller than the values of $\prod_{d} B_{d,2}$, which support our claim that $B^{\Jac}_{1:D}$ is a significantly tighter quantification then $\prod_{d} B_{d,2}$.

%

\bibliographystyle{abbrv}
\bibliography{capacity.bib}

\newpage

\appendix

\section{Proof of Theorem~\ref{thm:tighter_upperbd}}\label{pf:thm:tighter_upperbd}

We start with some definitions of notations. Given a vector $x \in \RR^p$, we denote $x_i$ as the $i$-th entry, and $x_{i:j}$ as a sub-vector indexed from $i$-th to $j$-th entries of $x$. Given a matrix $A \in \RR^{n \times m}$, we denote $A_{ij}$ as the entry corresponding to $i$-th row and $j$-th column, $A_{i*}$ ($A_{*i}$) as the $i$-th row (column), $A_{\cI_1 \cI_2}$ as a submatrix of $A$ indexed by the set of rows $\cI_1 \subseteq [n]$ and columns $\cI_2 \subseteq [m]$.
Given two real values $a,b \in \RR^+$, we write $a \lesssim(\gtrsim) b$ if $a \leq(\geq) c b$ for some generic constant $c>0$. 

Our analysis is based on the characterization of the Lipschitz property of a given function on both input and parameters. Such an idea can potentially provide tighter bound on the model capacity in terms of these Lipschitz constants and the number of free parameters, including other architectures of DNNs. 


We first provide an upper bound for the Lipschitz constant of $f\rbr{\cW_{D},x}$ in terms of the parameters $\cW_{D}$.

\begin{lemma}\label{lem:lipschitz_para}
	Given any $x \in \RR^{p_0}$ satisfying $\norm{x}_2 \leq R$, for any $f\rbr{\cW_{D},x},f\rbr{\tilde{\cW}_{D},x} \in \cF_{D,\Jac}$ with $\cW_{D} = \cbr{W_d}_{d=1}^D$ and $\tilde{\cW}_{D} = \cbr{\tilde{W}_d}_{d=1}^D$, where $W_d = U_d V_d^\top$ and $\tilde{W}_d = \tilde{U}_d \tilde{V}_d^\top$, $U_d, V_d, \tilde{U}_d, \tilde{V}_d \in \RR^{p \times r}$ and $\nbr{U}_2 = \nbr{V}_2 = \nbr{W_d}_2^{1/2}$, $\nbr{\tilde{U}}_2 = \nbr{\tilde{V}}_2 = \nbr{\tilde{W}_d}_2^{1/2}$, and denote $B^{\Jac, x}_{\backslash d} = \max_{d \in [D]} B^{\Jac, x}_{1:(d-1),2} B^{\Jac, x}_{(d+1):D,2}$, then we have
	\begin{align*}
		\nbr{f\rbr{\cW_{D},x} - f\rbr{\tilde{\cW}_{D},x}}_2 \leq B^{\Jac, x}_{\backslash d} \cdot R \sqrt{2D} \cdot \max_{d} B_{d,2}^{1/2} \sqrt{\sum_{d=1}^{D} \nbr{V - \tilde{V}}_{\Fr}^2 + \nbr{U - \tilde{U}}_{\Fr}^2 }.
	\end{align*}
\end{lemma}
\begin{proof}
	Given $x$ and two sets of weight matrices $\cbr{W_{d}}_{d=1}^D$, $\cbr{\tilde{W}_{d}}_{d=1}^D$, we have
	\begin{align}
		&\nbr{f_{W_{D}} \rbr{f_{W_{D-1}}\rbr{ \cdots f_{W_{1}} \rbr{x} } } - f_{\tilde{W}_{D}} \rbr{f_{\tilde{W}_{D-1}} \rbr{\cdots f_{\tilde{W}_{1}}\rbr{x} } } }_2 \nonumber \\
		&\overset{(i)}{=} \nbr{ \sum_{d=1}^{D} f_{W_{D}} \rbr{\cdots f_{W_{d+1}}\rbr{ f_{\tilde{W}_{d}}\rbr{ \cdots f_{\tilde{W}_{1}} \rbr{x} } } } - f_{W_{D}} \rbr{\cdots f_{W_{d}}\rbr{ f_{\tilde{W}_{d-1}}\rbr{ \cdots f_{\tilde{W}_{1}} \rbr{x} } } } }_2 \nonumber\\
		&\leq \sum_{d=1}^{D} \nbr{ f_{W_{D}} \rbr{\cdots f_{W_{d+1}}\rbr{ f_{\tilde{W}_{d}}\rbr{ \cdots f_{\tilde{W}_{1}} \rbr{x} } } } - f_{W_{D}} \rbr{\cdots f_{W_{d}}\rbr{ f_{\tilde{W}_{d-1}}\rbr{ \cdots f_{\tilde{W}_{1}} \rbr{x} } } } }_2 \nonumber\\
		&\overset{(ii)}{=} \sum_{d=1}^{D} \nbr{ J_{(d+1):D}^x \cdot f_{\tilde{W}_{d}}\rbr{ \cdots f_{\tilde{W}_{1}} \rbr{x} } - J_{(d+1):D}^x \cdot f_{W_{d}}\rbr{ f_{\tilde{W}_{d-1}}\rbr{ \cdots f_{\tilde{W}_{1}} \rbr{x} } } }_2 \nonumber\\
		&\overset{(iii)}{\leq} \sum_{d=1}^{D} B^{\Jac, x}_{(d+1):D} \cdot \nbr{\tilde{W}_d f_{\tilde{W}_{d-1}}\rbr{ \cdots f_{\tilde{W}_{1}} \rbr{x} } - W_d f_{\tilde{W}_{d-1}}\rbr{ \cdots f_{\tilde{W}_{1}} \rbr{x} } }_2 \nonumber \\
		&\leq \sum_{d=1}^{D} B^{\Jac, x}_{(d+1):D} \cdot \nbr{W_{d} - \tilde{W}_{d}}_{2} \cdot \nbr{f_{\tilde{W}_{d-1}}\rbr{ \cdots f_{\tilde{W}_{1}} \rbr{x} } }_2, \label{eqn:lip_w_1}
	\end{align}
	where $(i)$ is derived from adding and subtracting intermediate neural network functions recurrently, where $f_{W_{D}} \rbr{\cdots f_{W_{d+1}}\rbr{ f_{\tilde{W}_{d}}\rbr{ \cdots f_{\tilde{W}_{1}} \rbr{x} } } }$ share the same output of activation functions from $d+1$-th layer to $D$-the layer with $f_{W_{D}} \rbr{f_{W_{D-1}}\rbr{ \cdots f_{W_{1}} \rbr{x} } }$, $(ii)$ is from fixing the activation function output, and $(iii)$ is from the entry-wise $1$--Lipschitz continuity of $\sigma(\cdot)$. On the other hand, for any $d \in [D]$, we further have
	\begin{align}
		\nbr{ f_{W_{d}}\rbr{ \cdots f_{W_{1}} \rbr{x} } }_2 &= \nbr{ J_{1:d}^x \cdot x }_2 \leq B^{\Jac, x}_{1:d} \cdot \nbr{x}_2. \label{eqn:spec_f_d}
	\end{align}
	where $(i)$ is from the entry-wise $1$--Lipschitz continuity of $\sigma(\cdot)$ and $(ii)$ is from recursively applying the same argument. 
	
	In addition, since $W_d = U_d V_d^\top$ and $\tilde{W}_d = \tilde{U}_d \tilde{V}_d^\top$, where $U_d, V_d, \tilde{U}_d, \tilde{V}_d \in \RR^{p \times r}$ and $\nbr{U}_2 = \nbr{V}_2 = \nbr{\tilde{U}}_2 = \nbr{\tilde{V}}_2 = \nbr{W_d}_2^{1/2}$. Then we have
	\begin{align}
		\nbr{W_{d} - \tilde{W}_{d}}_{2} &= \nbr{U_d V_d^\top - \tilde{U}_d \tilde{V}_d^\top}_2  \nonumber \\
		&= \nbr{U_d V_d^\top - U_d \tilde{V}_d^\top + U_d \tilde{V}_d^\top - \tilde{U}_d \tilde{V}_d^\top}_2 \nonumber \\
		&\leq \nbr{U}_2 \nbr{V - \tilde{V}}_2 + \nbr{\tilde{V}}_2 \nbr{U - \tilde{U}}_2  \nonumber \\
		&\leq \nbr{W_d}_2^{1/2} \rbr{\nbr{V - \tilde{V}}_{\Fr} + \nbr{U - \tilde{U}}_{\Fr} }. \label{eqn:factor}
	\end{align}
	
	Applying \eqref{eqn:lip_w_1} recursively and combining \eqref{eqn:spec_f_d} and \eqref{eqn:factor}, we obtain the desired result as
	\begin{align*}
		&\nbr{f_{W_{D}} \rbr{f_{W_{D-1}}\rbr{ \cdots f_{W_{1}} \rbr{x} } } - f_{\tilde{W}_{D}} \rbr{f_{\tilde{W}_{D-1}} \rbr{\cdots f_{\tilde{W}_{1}}\rbr{x} } } }_{2} \\
		&\leq \sum_{d=1}^{D} B^{\Jac, x}_{(d+1):D} \cdot B^{\Jac, x}_{1:(d-1)} \cdot \nbr{x}_2 \cdot \nbr{W_d}_2^{1/2} \rbr{\nbr{V - \tilde{V}}_{\Fr} + \nbr{U - \tilde{U}}_{\Fr} } \\
		&\leq B^{\Jac, x}_{\backslash d} \cdot R \sqrt{D} \cdot \max_{d} \nbr{W_d}_2^{1/2} \sum_{d=1}^{D} \rbr{\nbr{V - \tilde{V}}_{\Fr} + \nbr{U - \tilde{U}}_{\Fr} }\\
		&\leq B^{\Jac, x}_{\backslash d} \cdot R \sqrt{2D} \cdot \max_{d} \nbr{W_d}_2^{1/2} \sqrt{\sum_{d=1}^{D} \nbr{V - \tilde{V}}_{\Fr}^2 + \nbr{U - \tilde{U}}_{\Fr}^2 }.
	\end{align*}
\end{proof}

\begin{lemma}\label{lem:cover_lip}
	Suppose $g(w,x)$ is $L_w$-Lipschitz over $w \in \RR^h$ with $\norm{w}_2\leq K$ and $\alpha = \sup_{g \in \cG,x\in \cX_{m}} \abr{g(w,x)}$. Then the ERC of $\cG = \cbr{g(w,x)}$ satisfies 
	\begin{align*}
		\cR_m \rbr{\cG} = \cO \rbr{\frac{\alpha \sqrt{h \log \frac{KL_w \sqrt{m}}{\alpha \sqrt{h}}}}{\sqrt{m}}}.
	\end{align*}
\end{lemma}
\begin{proof}
	For any $w_1,w_2 \in \RR^h$ and $\cX_{m} = \cbr{x_i}_{i=1}^m$, we consider the matric $\Delta\rbr{g_1,g_2} = \max_{x_i \in \cX_{m}} \abr{g_1(x_i) - g_2(x_i)}$, which satisfies
	\begin{align}
		\Delta\rbr{g_1,g_2} = \max_{x \in \cX_{m}} \abr{g_1(x) - g_2(x)} = \abr{g\rbr{w_1,x} - g\rbr{w_2,x} } \leq L_w \nbr{w_1 - w_2}_2. \label{eqn:metric}
	\end{align}
	Since $g$ is a parametric function with $h$ parameters, then we have the covering number of $\cG$ under the metric $\Delta$ in \eqref{eqn:metric} satisfies
	\begin{align*}
		\cN \rbr{\cG, \Delta, \delta } \leq \rbr{\frac{3KL_w}{\delta}}^h.
	\end{align*}
	Then using the standard Dudley's entropy integral bound on the ERC \cite{mohri2012foundations}, we have the ERC satisfies
	\begin{align}
		\cR_m \rbr{\cG} &\lesssim \inf_{\beta > 0} ~\beta + \frac{1}{\sqrt{m}} \int_{\beta}^{\sup_{g \in \cG} \Delta\rbr{g,0}} \sqrt{\log \cN \rbr{\cG, \Delta, \delta } } ~d \delta. \label{eqn_dnn:dudley}
	\end{align}
	Since we have
	\begin{align*}
		\alpha=\sup_{g \in \cG, x \in \cX_{m}} \Delta\rbr{g,0} = \sup_{g \in \cG, x \in \cX_{m}} \abr{g(w,x)}. 
	\end{align*}
	Then we have
	\begin{align*}
		\cR_m \rbr{\cG} &\lesssim \inf_{\beta > 0} ~\beta + \frac{1}{\sqrt{m}} \int_{\beta}^{\alpha} \sqrt{ h \log \frac{K L_w}{\delta} } ~d \delta \\
		&\leq \inf_{\beta > 0} ~\beta + \alpha\sqrt{\frac{h \log \frac{KL_w}{\beta}}{m}} \overset{(i)}{\lesssim} \frac{\alpha \sqrt{h \log \frac{KL_w \sqrt{m}}{\alpha\sqrt{h}}}}{\sqrt{m}},
	\end{align*}
	where $(i)$ is obtained by taking $\beta = \alpha \sqrt{h/m}$.
\end{proof}

By definition, we have $\alpha = \sup_{f \in \cF_{D,\Jac}, x \in \cX_{m} } g_{\gamma} \rbr{f\rbr{\cW_{D},x}}$. From $\frac{1}{\gamma}$-Lipschitz continuity of $g$ and bounded Jacobian, we also have
\begin{align}
	\alpha \leq \frac{R \cdot B_{1:D}^{\Jac}}{\gamma}. \label{eqn:alpha}
\end{align}

From Lemma~\ref{lem:lipschitz_para}, we have
\begin{align*}
	L_w \leq \frac{\max_{x \in \cX_{m}} B^{\Jac}_{\backslash d} \cdot R \sqrt{2D} \cdot \max_{d} \nbr{W_d}_2^{1/2}}{\gamma}.
\end{align*}
Moreover, when $p_d=p$ for all $d \in [D]$, we have
\begin{align*}
	K = \sqrt{\sum_{d=1}^{D} \nbr{W_d}_{\Fr}^2 } \leq \sqrt{p D} \cdot \max_{d} \nbr{W_d}_2.
\end{align*}
Combining the results above with Lemma~\ref{lem:cover_lip} and $h=2 D pr$, we have
\begin{align*}
	\cR_m \rbr{\cG} \lesssim \frac{\alpha \sqrt{h \log \frac{KL_w \sqrt{m}}{\alpha\sqrt{h}}}}{\sqrt{m}} \lesssim \frac{R \cdot B_{1:D}^{\Jac} \sqrt{D pr \log \frac{B^{\Jac}_{\backslash d} \cdot R \sqrt{Dm/r} \cdot \max_{d} \nbr{W_d}_2/\gamma}{\sup_{f \in \cF_{D,\Jac}, x \in \cX_{m} } g_{\gamma} \rbr{f\rbr{\cW_{D},x}}} }}{\gamma \sqrt{m}}.
\end{align*}

\section{Proof of Corollary~\ref{cor:norm_ind}}\label{pf:cor:norm_ind}

The analysis follows Theorem~\ref{thm:tighter_upperbd}, except that the bound for $\alpha$ in \eqref{eqn:alpha} satisfies
\begin{align*}
	\alpha \leq \min \cbr{b, \frac{R \cdot B_{1:D}^{\Jac}}{\gamma}},
\end{align*}
since $g$ satisfies $\abr{g} \leq b$ and $\frac{1}{\gamma}$-Lipschitz continuous. Then we have the desired result.

\section{Proof of Corollary~\ref{cor:cnn_bd}}\label{pf:cor:cnn_bd}

We first show that using unit norm filters for all $d \in [D]$ and $n_{d} \leq k_{d}$, we have
\begin{align}
	\nbr{W_{d}}_2 = \sqrt{\frac{k_{d}}{s_{d}}}, \label{eqn_dnn:cnn_bd}
\end{align}

First note that when $n_{d} = k_{d}$, due to the orthogonality of $\cbr{ w^{(d,j)} }_{j=1}^{k_{d}}$, for all $i,q \in [k_{d}]$, $i \neq q$, we have
\begin{align}
	\sum_{j=1}^{k_{d}} \rbr{w^{(d,j)}_{i}}^2 = 1 ~~\text{and}~~\sum_{j=1}^{k_{d}} w^{(d,j)}_{q} \cdot w^{(d,j)}_{i} = 0. \label{eqn_dnn:cnn_int1}
\end{align}

When $n_{d} = k_{d}$, we have for all $i \in [p_{d-1}]$, the diagonal entries of $W_{d}^{\top}W_{d}$ satisfy
\begin{align}
	\rbr{W_{d}^{\top}W_{d}}_{ii} = \sum_{j=1}^{k_{d}} \nbr{\rbr{W_{d}^{(j)}}_{*i}}_2 = \sum_{j=1}^{k_{d}} \sum_{h=1}^{\frac{k_{d}}{s_{d}}} \rbr{w^{(d,j)}_{\rbr{i\%s_{d}} + \rbr{h-1}s_{d} } }^2 \overset{(i)}{=} \frac{k_{d}}{s_{d}}. \label{eqn_dnn:cnn_diag}
\end{align}
where $(i)$ is from \eqref{eqn_dnn:cnn_int1}. For the off-diagonal entries of $W_{d}^{\top}W_{d}$, i.e., for $i \neq q$, $i,q \in [p_{d}]$, we have
\begin{align}
	\rbr{W_{d}^{\top}W_{d}}_{iq} &= \sum_{j=1}^{k_{d}} \rbr{W_{d}^{(j)}}_{*q}^{\top} \rbr{W_{d}^{(j)}}_{*i} \nonumber \\
	&= \sum_{j=1}^{k_{d}} \sum_{h=1}^{\frac{k_{d}}{s_{d}}} w^{(d,j)}_{\rbr{i\%s_{d}} + \rbr{h-1}s_{d} } \cdot w^{(d,j)}_{\rbr{q\%s_{d}} + \rbr{h-1}s_{d} } \overset{(i)}{=} 0, \label{eqn_dnn:cnn_offdiag}
\end{align}
where $(i)$ is from \eqref{eqn_dnn:cnn_int1}. Combining \eqref{eqn_dnn:cnn_diag} and \eqref{eqn_dnn:cnn_offdiag}, we have that $W_{d}^{\top}W_{d}$ is a diagonal matrix with
\begin{align*}
	\nbr{W_{d}^{\top}W_{d} }_2 = \frac{k_{d}}{s_{d}} ~~\Longrightarrow~~ \nbr{W_{d} }_2 = \sqrt{\frac{k_{d}}{s_{d}}}.
\end{align*}
For $n_{d} < n_{k}$, we have that $W_{d}$ is a row-wise submatrix of that when $n_{d} = k_{d}$, denoted as $\tilde{W}_{d}$. Let $S \in \RR^{\frac{n_{d}k_{d}}{s_{d}} \times p_{d}}$ be a row-wise submatrix of an identity matrix corresponding to sampling the row of $W_{d}$ to form $\tilde{W}_{d}$. Then we have that \eqref{eqn_dnn:cnn_bd} holds, and since
\begin{align*}
	\nbr{\tilde{W}_{d}}_2 = \sqrt{\nbr{S \cdot W_{d} W_{d}^{\top} \cdot S^{\top} }_2^2 } =  \sqrt{\frac{k_{d}}{s_{d}}}.
\end{align*}

Suppose $k_1 = \cdots = k_D = k$ for ease of discussion. Then following the same argument as in the proof of Theorem~\ref{thm:tighter_upperbd} and Lemma~\ref{lem:cover_lip}, we have 
\begin{align*}
	\alpha &= \sup_{f \in \cF_{D,\Jac}, x \in \cX_{m} } g_{\gamma} \rbr{f\rbr{\cW_{D},x}} \leq \frac{R \cdot \prod_{d=1}^D \nbr{W_d}_2}{\gamma} = \frac{R \cdot\prod_{d=1}^D \sqrt{\frac{k}{s_{d}}} }{\gamma}, \\
	L_w &\leq \max_{x \in \cX_{m}} B^{\Jac, x}_{\backslash d} \cdot R \sqrt{2Dk/s}, \\
	K &= \sqrt{\sum_{d=1}^{D} \sum_{j=1}^{n_d} \nbr{ w^{(d,j)} }_2^2 } = \sqrt{\sum_{d=1}^D n_{d}}, ~\text{and}~\\
	h &= k \sum_{d=1}^D n_{d}.
\end{align*}

Using the fact that the number of parameters in each layer is no more than $k n_{d}$ rather than $2pr$, we have
\begin{align*}
	\cR_m \rbr{\cG} &\lesssim \frac{\alpha \sqrt{h \log \frac{KL_w \sqrt{m}}{\alpha\sqrt{h}}}}{\sqrt{m}} \\
	&\lesssim \frac{R \cdot \prod_{d=1}^D \sqrt{\frac{k}{s_{d}}} \cdot \sqrt{k \sum_{d=1}^D n_{d} \log \frac{B^{\Jac}_{\backslash d} \cdot R \sqrt{Dm/s}/\gamma}{\sup_{f \in \cF_{D,\Jac}, x \in \cX_{m} } g_{\gamma} \rbr{f\rbr{\cW_{D},x}}} }}{\gamma \sqrt{m}}.
\end{align*}
We finish the proof by combining with Corollary~\ref{cor:norm_ind}.

\section{Proof of Corollary~\ref{cor:upperbd_resnet_indep}}\label{pf:cor:upperbd_resnet_indep}

The analysis is analogous to the proof for Theorem~\ref{thm:tighter_upperbd}, but with different construction of the intermediate results. 
We first provide an upper bound for the Lipschitz constant of $f\rbr{\cV_{D},\cU_{D},x}$ in terms of $\cV_{D}$ and $\cU_{D}$.

\begin{lemma}\label{lem:lipschitz_res_para}
	Given $x \in \RR^{p_0}$ with $\norm{x}_2 \leq R$, any $f\rbr{\cV_{D},\cU_{D},x},f\rbr{\tilde{\cV}_{D},\tilde{\cU}_{D},x} \in \cF_{D,\Jac}$ with $\cV_{D} = \cbr{V_d}_{d=1}^D$, $\cU_{D} = \cbr{U_d}_{d=1}^D$ ,$\tilde{\cV}_{D} = \cbr{\tilde{V}_d}_{d=1}^D$, and $\tilde{\cU}_{D} = \cbr{\tilde{U}_d}_{d=1}^D$, and denote $B^{\Jac, x}_{\backslash d} = \max_{d \in [D]} B^{\Jac, x}_{1:(d-1)} B^{\Jac, x}_{(d+1):D}$, then we have
	\begin{align*}
		&\nbr{f\rbr{\cV_{D},\cU_{D},x} - f\rbr{\tilde{V}_{D}, \tilde{U}_{D},x}}_2 \\
		&\leq B^{\Jac, x}_{\backslash d} \max_{d} \rbr{B_{V_d,2} + B_{U_d,2}} R\sqrt{2D} \cdot \sqrt{\sum_{d=1}^{D} \nbr{V_{d} - \tilde{V}_{d}}_{\rm F}^2 + \nbr{U_{d} - \tilde{U}_{d}}_{\rm F}^2 }.
	\end{align*}
\end{lemma}
\begin{proof}
	Given $x$ and two sets of weight matrices $\cbr{U_{d}, V_{d}}_{d=1}^D$, $\cbr{\tilde{U}_{d},\tilde{U}_{d}}_{d=1}^D$, we have
	\begin{align}
		&\nbr{f_{V_{D}, U_{D}} \rbr{f_{V_{D-1}, U_{D-1}}\rbr{ \cdots f_{V_{1}, U_{1}} \rbr{x} } } - f_{\tilde{V}_{D}, \tilde{U}_{D}} \rbr{f_{\tilde{V}_{D-1}, \tilde{U}_{D-1}} \rbr{\cdots f_{\tilde{V}_{1}, \tilde{U}_{1}}\rbr{x} } } }_2 \nonumber \\
		&\leq \sum_{d=1}^{D} \nbr{ f_{V_{D}, U_{D}} \rbr{\cdots f_{V_{d+1}, U_{d+1}}\rbr{ f_{\tilde{V}_{d}, \tilde{U}_{d}} \rbr{ \cdots  } } } - f_{V_{D}, U_{D}} \rbr{\cdots f_{V_{d+1}, U_{d+1}}\rbr{ f_{\tilde{V}_{d}, U_{d}} \rbr{ \cdots  } } } }_2 \nonumber \\
		&\hspace{0.2in} + \sum_{d=1}^{D} \nbr{ f_{V_{D}, U_{D}} \rbr{\cdots f_{V_{d+1}, U_{d+1}}\rbr{ f_{\tilde{V}_{d}, U_{d}} \rbr{ \cdots  } } } - f_{V_{D}, U_{D}} \rbr{\cdots f_{V_{d}, U_{d}}\rbr{ f_{\tilde{V}_{d-1}, \tilde{U}_{d-1}} \rbr{ \cdots } } } }_2 \nonumber \\
		&= \sum_{d=1}^{D} \nbr{ J_{(d+1):D}^x \cdot f_{\tilde{V}_{d}, \tilde{U}_{d}} \rbr{ f_{\tilde{V}_{d-1}, \tilde{U}_{d-1}} \rbr{ \cdots } } - J_{(d+1):D}^x \cdot f_{\tilde{V}_{d}, U_{d}} \rbr{ f_{\tilde{V}_{d-1}, \tilde{U}_{d-1}} \rbr{ \cdots } } }_2 \nonumber \\
		&\hspace{0.5in} + \sum_{d=1}^{D} \nbr{ J_{(d+1):D}^x \cdot f_{\tilde{V}_{d}, U_{d}} \rbr{ f_{\tilde{V}_{d-1}, \tilde{U}_{d-1}} \rbr{ \cdots } } - J_{(d+1):D}^x \cdot f_{V_{d}, U_{d}}\rbr{ f_{\tilde{V}_{d-1}, \tilde{U}_{d-1}} \rbr{ \cdots } } }_2 \nonumber \\
		&\leq \sum_{d=1}^{D} B^{\Jac, x}_{(d+1):D} \cdot \nbr{ f_{\tilde{V}_{d}, \tilde{U}_{d}} \rbr{ f_{\tilde{V}_{d-1}, \tilde{U}_{d-1}} \rbr{ \cdots } } - f_{\tilde{V}_{d}, U_{d}} \rbr{ f_{\tilde{V}_{d-1}, \tilde{U}_{d-1}} \rbr{ \cdots } } }_2 \nonumber \\
		&\hspace{0.5in} + \sum_{d=1}^{D} B^{\Jac, x}_{(d+1):D} \cdot \nbr{ f_{\tilde{V}_{d}, U_{d}} \rbr{ f_{\tilde{V}_{d-1}, \tilde{U}_{d-1}} \rbr{ \cdots } } - f_{V_{d}, U_{d}}\rbr{ f_{\tilde{V}_{d-1}, \tilde{U}_{d-1}} \rbr{ \cdots } } }_2 \nonumber \\
		&\overset{(i)}{\leq} \sum_{d=1}^{D} B^{\Jac, x}_{(d+1):D} \cdot \nbr{ \tilde{V}_{d} \sigma\rbr{\tilde{U}_{d} \cdot {f_{\tilde{V}_{d-1}, \tilde{U}_{d-1}}\rbr{ \cdots } } } - \tilde{V}_{d} \sigma\rbr{U_{d} \cdot {f_{\tilde{V}_{d-1}, \tilde{U}_{d-1}}\rbr{ \cdots } } } }_2 \nonumber \\
		&\hspace{0.5in} + \sum_{d=1}^{D} B^{\Jac, x}_{(d+1):D} \cdot \nbr{ \tilde{V}_{d} \sigma\rbr{U_{d} \cdot {f_{\tilde{V}_{d-1}, \tilde{U}_{d-1}}\rbr{ \cdots } } } - V_{d} \sigma\rbr{U_{d} \cdot {f_{\tilde{V}_{d-1}, \tilde{U}_{d-1}}\rbr{ \cdots } } } }_2 \nonumber \\
		&\overset{(ii)}{\leq} \sum_{d=1}^{D} B^{\Jac, x}_{(d+1):D} \cdot \rbr{\nbr{U_{d} - \tilde{U}_{d} }_2 \nbr{V_{d}}_2 + \nbr{V_{d} - \tilde{V}_{d} }_2 \nbr{U_{d}}_2 } \nbr{ f_{\tilde{V}_{d-1}, \tilde{U}_{d-1}}\rbr{ \cdots } }_2, \label{eqn:lip_res_vu1}
	\end{align}
	where we choose same activations from $d$-th to $D$-th layer in $f_{V_{D}, U_{D}} \rbr{\cdots f_{V_{d+1}, U_{d+1}}\rbr{ f_{\tilde{V}_{d}, U_{d}} \rbr{ \cdots  } } }$ with $f_{V_{D}, U_{D}} \rbr{ \cdots f_{V_{1}, U_{1}} \rbr{x} }$ and same activations from $1$-st to $d-1$-th layer with $f_{\tilde{V}_{D}, \tilde{U}_{D}} \rbr{\cdots f_{\tilde{V}_{1}, \tilde{U}_{1}}\rbr{x} }$,  $(i)$ and $(ii)$ from the entry-wise $1$--Lipschitz continuity of $\sigma(\cdot)$. In addition, for any $d \in [D]$, we further have
	\begin{align}
		\nbr{ f_{\tilde{V}_{d-1}, \tilde{U}_{d-1}}\rbr{ \cdots f_{\tilde{V}_{1}, \tilde{U}_{1}} \rbr{x} } }_2 &= \nbr{ J_{1:d}^x \cdot x }_2 \leq B^{\Jac, x}_{1:(d-1)} \cdot \nbr{x}_2. \label{eqn:lip_res_vu2}
	\end{align}
	
	Combining \eqref{eqn:lip_res_vu1} and \eqref{eqn:lip_res_vu2}, we obtain 
	\begin{align*}
		&\nbr{f_{V_{D}, U_{D}} \rbr{ \cdots f_{V_{1}, U_{1}} \rbr{x} } - f_{\tilde{V}_{D}, \tilde{U}_{D}} \rbr{\cdots f_{\tilde{V}_{1}, \tilde{U}_{1}}\rbr{x} } }_{2} \\
		&\leq \sum_{d=1}^{D} B^{\Jac, x}_{(d+1):D} \cdot B^{\Jac, x}_{1:(d-1)} \cdot \nbr{x}_2 \cdot \rbr{ \nbr{V_{d} - \tilde{V}_{d}}_{\rm F} \cdot \nbr{U_{d}}_2 + \nbr{U_{d} - \tilde{U}_{d}}_{\rm F} \cdot \nbr{\tilde{V}_{d}}_2 } \\
		&\leq B^{\Jac, x}_{\backslash d} \max_{d} \rbr{\nbr{V_d}_2 + \nbr{U_d}_2} R \sum_{d=1}^{D} \cdot \rbr{ \nbr{V_{d} - \tilde{V}_{d}}_{\rm F} + \nbr{U_{d} - \tilde{U}_{d}}_{\rm F} } \\
		&\leq B^{\Jac, x}_{\backslash d} \max_{d} \rbr{\nbr{V_d}_2 + \nbr{U_d}_2} R\sqrt{2D} \cdot \sqrt{\sum_{d=1}^{D} \nbr{V_{D} - \tilde{V}_{D}}_{\rm F}^2 + \sum_{d=1}^{D} \nbr{U_{D} - \tilde{U}_{D}}_{\rm F}^2 }.
	\end{align*}
\end{proof}


Let $p_1=\cdots=p_D=p$ and $q_1=\cdots=q_D=q$. Then following the same argument as in the proof of Theorem~\ref{thm:tighter_upperbd}, we have 
\begin{align*}
	\alpha &= \sup_{f \in \cF_{D,\Jac}, x \in \cX_{m} } g_{\gamma} \rbr{f\rbr{\cV_{D},\cV_{D},x}} \leq \frac{R \cdot B_{1:D}^{\Jac}}{\gamma}, \\
	L_w &\leq \max_{x \in \cX_{m}} B^{\Jac, x}_{\backslash d} \max_{d} \rbr{\nbr{V_d}_2 + \nbr{U_d}_2} R\sqrt{2D}, \\
	K &= \sqrt{\sum_{d=1}^{D} \nbr{ V_{d} }_F^2 + \nbr{U_{d} }_F^2} \leq \sqrt{p D} \cdot \max_{d} \rbr{\nbr{V_d}_2 + \nbr{U_d}_2}, ~\text{and}~h = 2 D pq,
\end{align*}

Combining Lemma~\ref{lem:cover_lip} and Lemma~\ref{lem:lipschitz_res_para}, we have
\begin{align*}
	&\cR_m \rbr{\cG} \lesssim \frac{\alpha \sqrt{h \log \frac{KL_w \sqrt{m}}{\alpha\sqrt{h}}}}{\sqrt{m}} \\
	&\lesssim \frac{R \cdot B_{1:D}^{\Jac} \cdot \sqrt{D pq \cdot \log \rbr{ \frac{B^{\Jac}_{\backslash d} \max_{d} \rbr{\nbr{V_d}_2 + \nbr{U_d}_2} R\sqrt{m/q}/\gamma }{\sup_{f \in \cF_{D,\Jac}, x \in \cX_{m} } g_{\gamma} \rbr{f\rbr{\cV_{D},\cV_{D},x}} } } } }{\gamma \sqrt{m} }.
\end{align*}
We finish the proof by combining with Corollary~\ref{cor:norm_ind}.

\section{Spectral Bound for $W_{d}$ in CNNs with Matrix Filters}\label{apx:mat_filter}

We provide further discussion on the upper bound of the spectral norm for the weight matrix $W_{d}$ in CNNs with matrix filters. In particular, by denoting $W_{d}$ using submatrices as in \eqref{eqn_dnn:cnn_weight}, i.e.,
\begin{align*}
	W_{d} = \sbr{W_{d}^{(1)\top}~\cdots~W_{d}^{(n_{d})\top}}^\top \in \RR^{p_{d} \times p_{d-1}},
\end{align*}
we have that each block matrix $W_{d}^{(j)}$ is of the form
\begin{align}
	W_{d}^{(j)} = \sbr{\begin{array}{cccc}
			W_{d}^{(j)}\rbr{1,1} & W_{d}^{(j)}\rbr{1,2} & \cdots & W_{d}^{(j)}\rbr{1,\sqrt{p_{d-1}}} \\
			W_{d}^{(j)}\rbr{2,1} & W_{d}^{(j)}\rbr{2,2} & \cdots & W_{d}^{(j)}\rbr{2,\sqrt{p_{d-1}}} \\
			\vdots & \vdots & \ddots & \vdots \\
			W_{d}^{(j)}\rbr{\frac{\sqrt{p_{d-1} k_{d}}}{s_{d}},1} & W_{d}^{(j)}\rbr{\frac{\sqrt{p_{d-1} k_{d}}}{s_{d}},2} & \cdots & W_{d}^{(j)}\rbr{\frac{\sqrt{p_{d-1} k_{d}}}{s_{d}},\sqrt{p_{d-1}}} \\
	\end{array} }, \label{eqn:block_W}
\end{align}
where $W_{d}^{(j)}\rbr{i,l} \in \RR^{\frac{\sqrt{p_{d-1} k_{d}}}{s_{d}} \times \sqrt{p_{d-1}}}$ for all $i \in \sbr{\frac{\sqrt{p_{d-1} k_{d}}}{s_{d}}}$ and $l \in \sbr{\sqrt{p_{d-1}}}$. Particularly, off-diagonal blocks are zero matrices, i.e., $W_{d}^{(j)}\rbr{i,l} = 0$ for $i \neq l$. For diagonal blocks, we have
\begin{align}
	W_{d}^{(j)} \rbr{i,i}= \sbr{\begin{array}{c}
			\underbrace{w^{(j,1)}}_{\in \RR^{\sqrt{k_{d}}}} \hspace{-0.0in}\underbrace{0 \cdots\cdots\cdots\cdots 0}_{\in \RR^{\sqrt{\frac{p_{d-1}}{k_{d}}}-\sqrt{k_{d}}}}\hspace{-0.0in} \cdots\cdots\cdots\cdots\cdots
			\underbrace{w^{(j,\sqrt{k_{d}})}}_{\in \RR^{\sqrt{k_{d}}}} \hspace{-0.0in}\underbrace{0 \cdots\cdots\cdots\cdots\cdots\cdots\cdots 0}_{\in \RR^{\sqrt{\frac{p_{d-1}}{k_{d}}}-\sqrt{k_{d}}}} \\
			\underbrace{0\cdots 0}_{\in \RR^{\frac{s_{d}}{\sqrt{k_{d}}}}} \underbrace{w^{(j,1)}}_{\in \RR^{\sqrt{k_{d}}}} \hspace{-0.0in}\underbrace{0 \cdots\cdots\cdots\cdots 0}_{\in \RR^{\sqrt{\frac{p_{d-1}}{k_{d}}}-\sqrt{k_{d}}}}\hspace{-0.0in} \cdots\cdots\cdots\cdots\cdots
			\underbrace{w^{(j,\sqrt{k_{d}})}}_{\in \RR^{\sqrt{k_{d}}}} \hspace{-0.0in}\underbrace{0 \cdots\cdots\cdots\cdots\cdot\cdot 0}_{\hspace{0.0in}\in \RR^{\sqrt{\frac{p_{d-1}}{k_{d}}}-\sqrt{k_{d}}-\frac{s_{d}}{\sqrt{k_{d}}}}} \\
			\vdots \\
			w^{(j,1)}_{\cbr{\frac{s_{d}}{\sqrt{k_{d}}}}} \hspace{-0.0in}\underbrace{0 \cdots\cdots\cdots\cdots 0}_{\in \RR^{\sqrt{\frac{p_{d-1}}{k_{d}}}-\sqrt{k_{d}}}}\hspace{-0.0in} \cdots\cdots\cdots
			\underbrace{w^{(j,\sqrt{k_{d}})}}_{\in \RR^{\sqrt{k_{d}}}} \hspace{-0.0in}\underbrace{0 \cdots\cdots\cdots\cdots\cdots\cdots\cdot\cdot 0}_{\in \RR^{\sqrt{\frac{p_{d-1}}{k_{d}}}-\sqrt{k_{d}}}} w^{(j,1)}_{\cbr{\frac{s_{d}}{1}}}
	\end{array} }.\label{eqn:block_W_diag}
\end{align}
where $w^{(j,1)}_{\cbr{\frac{s_{d}}{1}}} = w^{(j,1)}_{1:\frac{s_{d}}{\sqrt{k_{d}}}}  \in \RR^{\frac{s_{d}}{\sqrt{k_{d}}}}$ and $w^{(j,1)}_{\cbr{\frac{s_{d}}{\sqrt{k_{d}}}}} = w^{(j,1)}_{\rbr{\sqrt{k_{d}} - \frac{s_{d}}{\sqrt{k_{d}}} + 1}: \sqrt{k_{d}}} \in \RR^{\frac{s_{d}}{\sqrt{k_{d}}}}$. Combining \eqref{eqn:block_W} and \eqref{eqn:block_W_diag}, we have that the stride for $W_{d}^{(j)}$ is $\frac{s_{d}^2}{k_{d}}$. Using the same analysis for Corollary~\ref{cor:cnn_bd}. We have $\nbr{W_{d}}_2 = 1$ if $\sqrt{\sum_{i}\nbr{w^{(j,i)}}_2^2} = \frac{k_{d}}{s_{d}}$.

For image inputs, we need an even smaller matrix $W_{d}^{(j)} \rbr{i,i}$ with fewer rows than \eqref{eqn:block_W_diag}, denoted as
\begin{align}
	W_{d}^{(j)} \rbr{i,i}= \sbr{\begin{array}{c}
			\underbrace{w^{(j,1)}}_{\in \RR^{\sqrt{k_{d}}}} \hspace{-0.0in}\underbrace{0 \cdots\cdots\cdots\cdots 0}_{\in \RR^{\sqrt{\frac{p_{d-1}}{k_{d}}}-\sqrt{k_{d}}}}\hspace{-0.0in} \cdots\cdots\cdots\cdots\cdots
			\underbrace{w^{(j,\sqrt{k_{d}})}}_{\in \RR^{\sqrt{k_{d}}}} \hspace{-0.0in}\underbrace{0 \cdots\cdots\cdots\cdots\cdots\cdots\cdots 0}_{\in \RR^{\sqrt{\frac{p_{d-1}}{k_{d}}}-\sqrt{k_{d}}}} \\
			\underbrace{0\cdots 0}_{\in \RR^{\frac{s_{d}}{\sqrt{k_{d}}}}} \underbrace{w^{(j,1)}}_{\in \RR^{\sqrt{k_{d}}}} \hspace{-0.0in}\underbrace{0 \cdots\cdots\cdots\cdots 0}_{\in \RR^{\sqrt{\frac{p_{d-1}}{k_{d}}}-\sqrt{k_{d}}}}\hspace{-0.0in} \cdots\cdots\cdots\cdots\cdots
			\underbrace{w^{(j,\sqrt{k_{d}})}}_{\in \RR^{\sqrt{k_{d}}}} \hspace{-0.0in}\underbrace{0 \cdots\cdots\cdots\cdots\cdot\cdot 0}_{\hspace{0.0in}\in \RR^{\sqrt{\frac{p_{d-1}}{k_{d}}}-\sqrt{k_{d}}-\frac{s_{d}}{\sqrt{k_{d}}}}} \\
			\vdots \\
			\hspace{-0.0in}\underbrace{0 \cdots\cdots\cdots\cdots\cdots\cdots 0}_{\in \RR^{\sqrt{\frac{p_{d-1}}{k_{d}}}-\sqrt{k_{d}}}} \underbrace{w^{(j,1)}}_{\in \RR^{\sqrt{k_{d}}}} \hspace{-0.0in}\underbrace{0 \cdots\cdots\cdots\cdots\cdots 0}_{\in \RR^{\sqrt{\frac{p_{d-1}}{k_{d}}}-\sqrt{k_{d}}}}\hspace{-0.0in} \cdots\cdots\cdots\cdots\cdots
			\underbrace{w^{(j,\sqrt{k_{d}})}}_{\in \RR^{\sqrt{k_{d}}}}
	\end{array} }.\label{eqn:block_W_diag2}
\end{align}
Then $\nbr{W_{d}}_2 \leq 1$ still holds if $\sqrt{\sum_{i}\nbr{w^{(j,i)}}_2^2} = \frac{k_{d}}{s_{d}}$ since $W_{d}$ generated using \eqref{eqn:block_W_diag2} is a submatrix of $W_{d}$ generated using \eqref{eqn:block_W_diag}.

\end{document}